﻿

\documentclass[a4paper,fleqn]{cas-sc}

\usepackage[T1]{fontenc}
\usepackage[numbers,sort&compress]{natbib}
\usepackage{amsthm}
\usepackage{mathtools}
\usepackage[normalem]{ulem}
\usepackage{float}
\restylefloat{figure}
\usepackage{placeins}
\usepackage{tabularx}

\ExplSyntaxOn
\AddToHook{env/figure/begin}{\skip_set:Nn \l_fig_abovecap_skip { 3pt }}
\ExplSyntaxOff

\ExplSyntaxOn
\cs_set:Npn \__make_tbl_caption:nn #1#2
{
  \l_tbl_align_tl
  \skip_vertical:N \l_tbl_abovecap_skip
  \hbox_set:Nn \l_tmpa_box {\sffamily\small\textbf{#1}}
  \hbox_set:Nn \l_tmpb_box {\sffamily\small#2}
  \dim_set:Nn \l_tmpa_dim
    {
      \dim_min:nn {\l_tbl_width_dim}
        {\dim_max:nn {\box_wd:N \l_tmpa_box} {\box_wd:N \l_tmpb_box}}
    }
  {\makebox[\dim_use:N \l_tbl_width_dim][c]
    {\parbox{\dim_use:N \l_tmpa_dim}
      {\raggedright\sffamily\small\textbf{\color{scolor}#1}\par#2\par\vskip4pt}}}
  \skip_vertical:N \l_tbl_belowcap_skip
}
\ExplSyntaxOff

\ExplSyntaxOn
\RenewDocumentCommand \orcidauthor { m m }
  {
    \seq_gput_right:Nn \g_stm_orcid_seq
      { { \ttfamily \tl_to_str:n { #1 } } }
  }
\ExplSyntaxOff

\newcommand{\name}{DLDMF}
\makeatletter
\@ifundefined{theorem}{\newtheorem{theorem}{Theorem}[section]}{}
\@ifundefined{proposition}{\newtheorem{proposition}[theorem]{Proposition}}{}
\@ifundefined{corollary}{\newtheorem{corollary}[theorem]{Corollary}}{}
\@ifundefined{lemma}{}{}
\@ifundefined{assumption}{}{}
\@ifundefined{remark}{}{}
\makeatother

\providecommand{\eps}{\varepsilon}
\providecommand{\epsden}{\varepsilon_{\mathrm{den}}}
\providecommand{\norm}[1]{\left\lVert #1 \right\rVert}
\providecommand{\inner}[2]{\left\langle #1,#2 \right\rangle}
\providecommand{\dd}{\mathrm{d}}

\begin{document}
\let\WriteBookmarks\relax
\renewcommand{\topfraction}{0.85}
\renewcommand{\bottomfraction}{0.65}
\renewcommand{\textfraction}{0.1}
\renewcommand{\floatpagefraction}{0.3}
\setcounter{topnumber}{4}
\setcounter{bottomnumber}{2}
\setcounter{totalnumber}{6}
\shorttitle{Disentangled Latent Dynamics Manifold Fusion}
\shortauthors{Z. Liang et al.}

\title [mode = title]{Parameterized Temporal Extrapolation of PDEs via Disentangled Latent Dynamics Manifold Fusion}

\author[1]{Zhangyong Liang}[orcid=0000-0003-4463-6433]
\author[2]{Jihong Zhu}
\author[3]{Huanhuan Gao}
\cormark[1]
\ead{gao_huanhuan@jlu.edu.cn}
\affiliation[1]{organization={National Center for Applied Mathematics, Tianjin University},
                city={Tianjin},
                postcode={300072},
                country={PR China}}

\affiliation[2]{organization={School of Mechanical Engineering, Northwestern Polytechnical
University},
                city={Xi’an},
                postcode={710072},
                country={China}}
                
\affiliation[3]{organization={School of Mechanical and Aerospace Engineering, Jilin University},
                city={Changchun},
                postcode={130025},
                country={China}}
\cortext[cor1]{Corresponding author}

\begin{abstract}
Parameterized temporal extrapolation of Partial Differential Equation (PDE) solutions requires neural surrogates to capture spatial dependence, parameter variation, physical constraints, and dynamics beyond the training interval.
However, parameter-dependent changes in solution dynamics challenge existing methods.
Collocation-based Physics-Informed Neural Networks (PINNs) treat time as an input coordinate, whereas auto-decoding-based continuous-time latent models require instance-specific optimization at inference.
To address these limitations, we propose \textbf{D}isentangled \textbf{L}atent \textbf{D}ynamics \textbf{M}anifold \textbf{F}usion (\textbf{DLDMF}), a pure physics framework based on space--time--parameter disentanglement and dynamic manifold fusion.
DLDMF maps PDE parameters through a feed-forward encoder that directly initializes and conditions a continuous latent state governed by a parameter-conditioned Neural Ordinary Differential Equation (Neural ODE).
The dynamic manifold fusion stage combines the spatial representation, parameter embedding, and evolving latent state within a shared nonlinear decoder to reconstruct parameterized spatiotemporal solutions.
Once trained on a fixed time interval, DLDMF predicts solutions for unseen parameters and extrapolates beyond the training horizon by integrating the learned latent dynamics, without parameter-specific retraining or test-time latent optimization.
Under explicit regularity assumptions, we derive conditional finite-horizon bounds for parameter perturbations, latent-flow and reconstruction errors, and numerical integration.
Experiments cover parameter interpolation, temporal extrapolation, parameter extrapolation, and joint parameter--time extrapolation.
Across the evaluated benchmarks, DLDMF demonstrates state-of-the-art performance for in-parameter temporal extrapolation and extends effective prediction to the more challenging joint out-of-parameter temporal extrapolation regime.
\end{abstract}


\maketitle


\section{Introduction}
\label{sec:intro}

Scientific Machine Learning (SML) increasingly provides data-assisted surrogates for families of Partial Differential Equation (PDE) solutions.
Classical numerical discretization methods, including finite difference~\cite{leveque2007finite}, finite element~\cite{hughes2003finite}, finite volume~\cite{moukalled2015finite}, and spectral methods~\cite{canuto2006spectral}, remain foundational tools for high-fidelity PDE simulation.
Repeated high-fidelity solves can be computationally expensive in parameter studies, inverse problems, and multi-query design, motivating learning-based surrogates that combine data with governing equations.
Extensive efforts have focused on embedding physical laws directly into neural networks.
Integrating physical constraints into neural architectures yields physics-regularized and structure-preserving models.
Representative examples include physics-informed and conservation-law-constrained residual minimization~\cite{raissi2019physics,lee2021deep}, Hamiltonian or Lagrangian architectures~\cite{Greydanus2019hnn,cranmer2020lag,lee2021machine}, and symmetry-aware models~\cite{satorras2021en}.
Among these, Collocation-based Physics-Informed Neural Networks (PINNs)~\cite{raissi2019physics} have become a standard framework because their residual formulation is compatible with automatic differentiation~\cite{baydin2018automatic} and modern deep learning libraries.
By concurrently minimizing PDE residuals and initial and boundary condition losses, PINNs can approximate PDE solutions without paired interior solution data.
PINN applications include complex time-evolving PDE systems~\cite{shukla2021parallel}, particularly fluid-dynamics problems in which nonlinear transport~\cite{jagtap2020conser} and complex space--time dynamics~\cite{jagtap2020ext} challenge optimization.

Physics-informed approximations also include variational and multiscale formulations.
Energy-based objectives encode variational mechanics~\cite{samaniego2020energy}, while $hp$-VPINNs combine neural trial spaces with local test functions and domain decomposition~\cite{kharazmi2021hpvpinns}.
Gradient-enhanced residuals incorporate derivative information~\cite{yu2022gradient}, and Fourier-feature analyses address high-frequency and multiscale bias~\cite{wang2021eigenvector}.
Parameterized PDEs have also been addressed through reduced and neural operators.
Operator inference learns low-dimensional dynamics from full-order snapshots~\cite{peherstorfer2016operator}, whereas reduced operator inference incorporates known PDE structures~\cite{qian2022reduced}.
POD-DL-ROM combines linear compression with a nonlinear parameter--time map~\cite{fresca2022poddlrom}, and implicit Fourier neural operators address heterogeneous mechanics~\cite{you2022ifno}.
DLDMF additionally targets temporal evolution beyond the training interval.

\begin{figure}
\centering
\includegraphics[width=.85\linewidth]{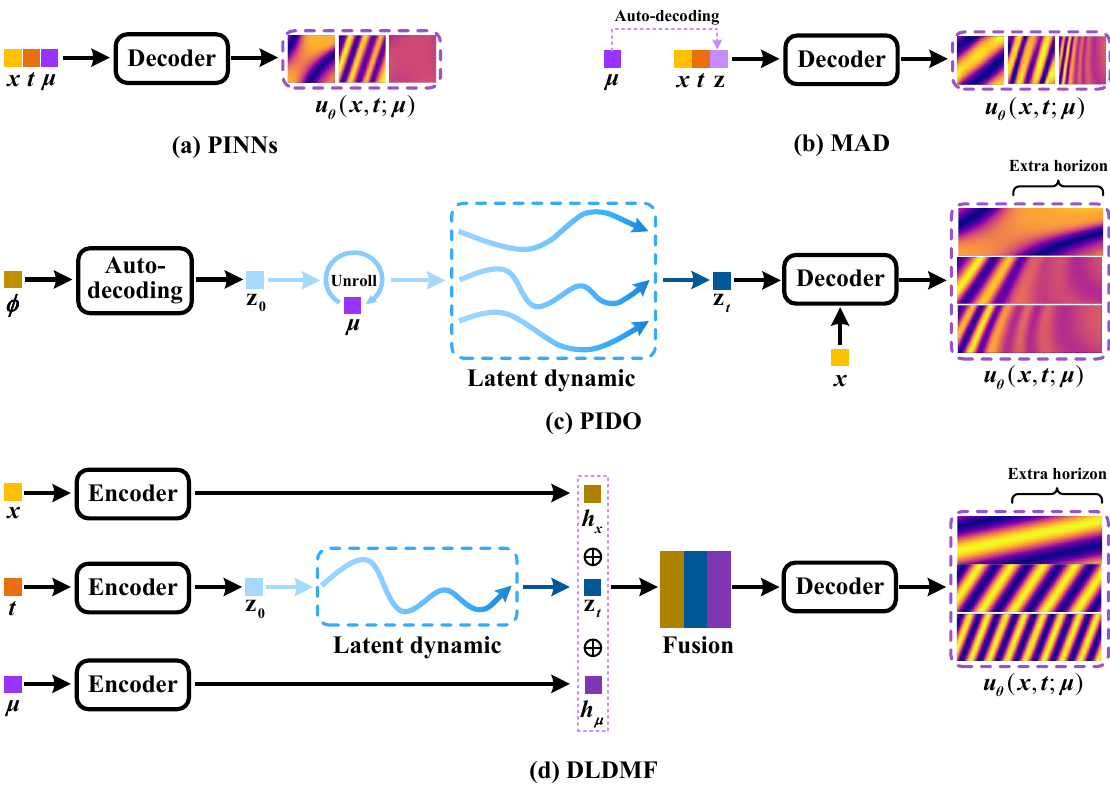}
\caption{Architectural comparison for parameterized PDEs.
(a) PINNs: direct solution prediction. (b) MAD: auto-decoded latent representation. (c) PIDO: auto-decoded latent dynamics. (d) DLDMF: parameter-initialized and conditioned latent dynamics with concatenated spatial, temporal, and parameter representations.}
\label{fig:framework}
\end{figure}

In the standard form, PINNs are trained for a single configuration of PDE parameters, boundary conditions, and geometries, limiting applicability in parametric settings.
To reduce the retraining burden of standard PINNs, recent works have developed parameterized models for entire PDE families within a unified framework.
Representative strategies include hypernetwork-based PINNs~\cite{liu2022novel}, meta-learning approaches such as auto-decoding~\cite{park2019deepsdf} and Meta-Auto-Decoding (MAD)~\cite{huang2022meta}, along with operator-learning methods that directly map parameterized inputs to solution spaces~\cite{li2020fourier,lu2021learning}.
Among parameterized architectures, Parameterized Physics-Informed Neural Networks (P$^2$INNs)~\cite{cho2024parameterized} provide a representative formulation.
By explicitly encoding PDE parameters into a latent representation, the embedding parameterizes the underlying solution network.
Rather than treating PDE parameters as additional input coordinates, P$^2$INNs construct a dedicated parameterized solution manifold.
The architectural choice allows a single trained model to query unseen parameter configurations within the evaluated parameter domain without retraining from scratch.
Prior results show competitive predictive accuracy for nonlinear PDEs over the parameter regimes considered~\cite{cho2024parameterized}.

Although P$^2$INNs address parameter generalization, the collocation-based formulation does not explicitly model temporal dynamics.
In P$^2$INNs, time is treated as an additional coordinate, and the solution network is optimized over a fixed temporal window.
Predictions beyond the training horizon rely on coordinate extrapolation and can degrade as the temporal shift increases.
The absence of explicit temporal dynamics restricts finite-horizon forecasting
and dynamical analysis, as properly capturing continuous-time evolution is as critical
as spatial representation for robust parameter generalization.
Moving structures challenge fixed linear approximations in transport-dominated systems.
Shifted POD aligns transported snapshots~\cite{reiss2018shifted}, while adaptive bases update local reduced spaces as structures propagate~\cite{peherstorfer2020transport}.
Transformed snapshot interpolation aligns parameter-dependent jumps before approximation~\cite{welper2017interpolation}.
Lagrangian PINNs relate convection-dominated training difficulty to the Kolmogorov $n$-width and reformulate evolution along characteristics~\cite{mojgani2023lagrangian}.

To introduce intrinsic temporal structure, recent works have proposed modeling PDE dynamics directly within the latent space.
Dynamics-Aware Implicit Neural Representations (DINO)~\cite{yin2023dino} and Physics-Informed Dynamics Representation Learner (PIDO)~\cite{wang2025pido} introduce Neural Ordinary Differential Equations (Neural ODEs)~\cite{chen2018neural} to characterize continuous-time evolution within a low-dimensional latent space.
By learning an underlying latent flow rather than independently regressing solution values at discrete time points, DINO and PIDO provide continuous-time models that have been evaluated over finite prediction horizons.
Such latent-dynamics approaches commonly infer latent states from initial or early observations.
The DINO and PIDO designs are primarily aligned with initial-value generalization.
Dependence on initial-condition encoding makes DINO and PIDO less directly aligned with parameterized settings in which the dominant source of variability is the PDE parameter vector rather than multiple observed initial trajectories.

The present work considers a fully amortized alternative to parameterized latent models that require instance-wise latent refinement.
We introduce \textbf{D}isentangled \textbf{L}atent \textbf{D}ynamics \textbf{M}anifold \textbf{F}usion (\textbf{DLDMF}), a physics-informed latent time-integration surrogate designed for parameter generalization and finite-horizon temporal extrapolation.
The key strategy of DLDMF is to disentangle parameter encoding, spatial representation, and latent temporal evolution before solution reconstruction, as illustrated in Fig.~\ref{fig:framework}.
Here, disentanglement denotes separate representation channels for spatial coordinates, PDE parameters, and latent temporal dynamics.
In this architecture, manifold fusion denotes the representation-level combination of these channels.
Physical constraints enter through the training objective.
Standard PINNs (Fig.~\ref{fig:framework}(a)) treat time and parameters as standard input coordinates.
Methods like MAD (Fig.~\ref{fig:framework}(b)) obtain latent codes via test-time auto-decoding, whereas continuous-time models like PIDO (Fig.~\ref{fig:framework}(c)) rely on auto-decoding initial observations to unroll latent trajectories.
DLDMF (Fig.~\ref{fig:framework}(d)) explicitly disentangles spatial coordinates, temporal dynamics, and PDE parameters through dedicated representation channels and uses feed-forward parameter-conditioned initialization for direct latent-state construction.
Specifically, feed-forward maps generate a spatial embedding $\pmb{h}_x$, a parameter embedding $\pmb{h}_\mu$, and an initial temporal state $\pmb{z}_0$.
The same parameter embedding conditions the Neural ODE vector field~\cite{chen2018neural}, whose integration produces $\pmb{z}_t$ at requested times.
A shared decoder combines $\pmb{h}_x$, $\pmb{z}_t$, and $\pmb{h}_\mu$ to reconstruct the solution.
The fusion interface separates the three channels while allowing nonlinear interactions during reconstruction.

Our main contributions are summarized as follows:
\begin{itemize}
    \item \textbf{Dynamic manifold fusion:}
    We propose DLDMF, a physics-informed architecture that disentangles spatial
    coordinates, PDE parameters, and temporal evolution through $\pmb{h}_x$,
    $\pmb{h}_\mu$, and $\pmb{z}_t(\pmb{\mu})$, respectively.
    Dynamic manifold fusion integrates these representations through a shared
    nonlinear decoder for parameterized spatiotemporal reconstruction.
    The fusion preserves distinct representations while capturing the nonlinear interactions.
    \item \textbf{Parameterized temporal extrapolation:}
    DLDMF maps PDE parameters directly to the latent initial state through
    $\pmb{z}_0(\pmb{\mu})=g_{\theta_0}(g_{\theta_p}(\pmb{\mu}))$.
    The same parameter embedding conditions the latent vector field, enabling
    the pretrained model to predict solutions at unseen parameter values and extrapolate beyond
    the training horizon through latent-flow integration.
    This joint conditioning unifies parameter and temporal extrapolation.
    \item \textbf{Finite-horizon error bounds and convergence:}
    Under explicit regularity and stability assumptions, we establish
    finite-horizon continuity and derive parameter-perturbation, latent-flow,
    reconstruction-error, and numerical-integration bounds.
    We further establish the convergence of the latent trajectory and
    decoded solution, identifying latent stability and representation accuracy 
    as key factors governing finite-horizon prediction.
    \item \textbf{Parameter-efficient decoder modulation:}
    We investigate activation- and weight-space modulation strategies for
    adaptation to target PDE parameters while retaining the pretrained
    parameter encoder and latent dynamics.
    The structured variants, including singular-direction modulation, reduce
    the number of trainable variables and provide compact decoder adaptation relative
    to full decoder fine-tuning.
    \item \textbf{Out-parameter temporal extrapolation generalization:}
    We evaluate DLDMF across parameter interpolation, temporal extrapolation,
    parameter extrapolation, and joint parameter--time extrapolation settings.
    The results demonstrate accurate parameter-conditioned latent evolution
    over extended horizons and preservation of solution structures under the
    evaluated out-of-parameter temporal shifts.
\end{itemize}

\section{Related work}
\label{sec:work}

\paragraph{\textbf{Neural PDE solvers.}}
Classical PDE solvers such as finite difference and finite element methods are accurate but can be costly under dense refinement~\cite{Olver2014}, motivating neural surrogates for efficient PDE approximation~\cite{karn2021phy}.
Early Galerkin- or Ritz-type neural formulations~\cite{rudd2015const} led to Physics-Informed Neural Networks (PINNs)~\cite{raissi2019physics}, which enforce PDE residuals at collocation points and impose initial and boundary conditions through differentiable loss terms~\cite{raissi2019physics,yu2018deep}.
The mesh-free formulation leverages automatic differentiation~\cite{baydin2018automatic}, reduces dependence on paired ground-truth data, and has been applied to fluid dynamics~\cite{shukla2021parallel,jagtap2020ext} and transport problems~\cite{jagtap2020conser}.
More broadly, physical inductive biases can be introduced through Hamiltonian~\cite{Greydanus2019hnn} or Lagrangian architectures~\cite{cranmer2020lag,lee2021machine} or through constrained optimization~\cite{rudd2015const}.
Despite advances in physics-informed PDE modeling, PINNs remain sensitive to stiff nonlinear dynamics and challenging optimization landscapes~\cite{krishn2021char,wang2021under}, motivating work on robustness~\cite{wang2022when,yao2023multi} and faster convergence~\cite{krishn2021char}.

\paragraph{\textbf{Learning parameterized PDEs.}}
Standard PINNs require retraining when parameters, initial conditions, or geometries change~\cite{liu2022novel}.
Operator-learning methods learn maps between function spaces~\cite{li2020fourier,lu2021learning}; Fourier Neural Operators and DeepONets are representative examples whose accuracy and computational cost depend on the available training data, discretization protocol, and input representation~\cite{li2020fourier,lu2021learning,lanth2022non}.
Physics-informed variants such as PI-DeepONet reduce dependence on paired solution data by incorporating residual losses~\cite{wang2021learn,lanth2022non}.
Meta-learning methods such as MAD instead optimize task-specific latent variables without an encoder~\cite{huang2022meta}.
Within the PINN family, P$^2$INNs encode continuous PDE parameters into a latent representation to form a parameterized solution manifold, enabling one model to adapt across parameter configurations~\cite{cho2024parameterized}.
The manifold formulation improves robustness compared with treating parameters as raw coordinates in the regimes examined by Cho et al.~\cite{cho2024parameterized}.

\paragraph{\textbf{Nonlinear-manifold and sequential reduced dynamics.}}
Nonlinear-manifold ROMs replace fixed linear spaces with learned manifolds.
Manifold Galerkin and manifold LSPG project dynamics onto an autoencoder manifold~\cite{lee2020manifold}.
Physics-informed variants add hyper-reduction and error estimation~\cite{kim2022masked}, while Neural Galerkin evolves neural approximations with active learning~\cite{bruna2024neural}.
DLDMF instead performs an amortized mapping from PDE parameters to the latent initial state and vector field.

\paragraph{\textbf{Implicit neural representations.}}
Implicit Neural Representations (INRs) map coordinates to continuous signal values~\cite{sitzmann2020inr,fathony2020multi}, giving PDE solvers grid-agnostic fields queryable at arbitrary locations~\cite{sitzmann2020inr} and casting PINNs as physics-informed spatiotemporal INRs~\cite{raissi2019physics}.
Multi-sequence INR variants condition on initial states to represent multiple trajectories~\cite{yu2022gen,skorok2022sty}, but they learn static initial-condition-to-trajectory maps over fixed horizons rather than intrinsic flows~\cite{fresca2020deep,chen2023crom}, limiting temporal extrapolation~\cite{yin2023dino}.
Explicit dynamics modeling instead learns temporal derivatives~\cite{green2019learn}; autoregressive schemes suffer from step-size constraints and accumulated error~\cite{brands2021mes}, whereas Neural ODEs provide continuous vector fields integrable at arbitrary times~\cite{chen2018neural,quag2019snode}.
DINO and PIDO infer latent initial states from initial or early observations~\cite{yin2023dino,wang2025pido}.
PNODE and HyperPNODE condition latent Neural ODEs on PDE parameters and couple them with INR decoders and optional physics-informed correction~\cite{wen2023reduced}.

Against existing methods, DLDMF maps PDE parameters directly to the latent initial state and vector field, forming an amortized parameter-conditioned latent dynamics model, as summarized in Table~\ref{tab:closest_work_structure}.

\begin{table}[t]
\centering
\scriptsize
\caption{Structural comparison of representative parameterized PDE methods by conditioning, dynamics, and inference.}
\label{tab:closest_work_structure}
\begin{tabularx}{0.98\linewidth}{@{}c*{4}{>{\centering\arraybackslash}X}@{}}
\toprule
Method & Conditioning mechanism & Latent dynamics & Training objective & Direct test-time inference \\
\midrule
DINO~\cite{yin2023dino} & Auto-decoding & Yes & Data fitting & No \\
PIDO~\cite{wang2025pido} & Auto-decoding & Yes & Physics-informed & No \\
MAD~\cite{huang2022meta} & Auto-decoding & No & Physics-informed & No \\
P$^2$INN~\cite{cho2024parameterized} & Parameter encoder & No & Physics-informed & Yes \\
\textbf{DLDMF} & Parameter encoder & Yes & Physics-informed & Yes \\
\bottomrule
\end{tabularx}
\par\vspace{2pt}
\begin{minipage}{0.98\linewidth}
\footnotesize\raggedright\textit{Note:} Latent dynamics denotes explicit latent-state evolution; direct inference denotes prediction using the pretrained model and feed-forward latent initialization.
\end{minipage}
\end{table}

\section{Problem formulation}
\label{sec:prob}

This section formalizes the learning setting and extrapolation regimes for time-dependent parameterized PDEs.
The section then reviews the auto-decoding pipeline and identifies the iterative inference step that motivates a feed-forward, parameter-conditioned latent initializer.

\subsection{Parameterized PDEs}
\label{sec:ppde}

Let $\Omega\subset\mathbb{R}^{d_x}$ be a bounded spatial domain, let $[0,T_{\mathrm{te}}]$ be a finite time interval, and let $\mathcal{P}_{\mathrm{eval}}\subset\mathbb{R}^{d_\mu}$ be a compact PDE-parameter evaluation domain.
For each $\pmb{\mu}\in\mathcal{P}_{\mathrm{eval}}$, we consider an evolution problem
\begin{equation}
\mathcal{N}_{\pmb{\mu}}[u](\pmb{x},t)=0,
\qquad (\pmb{x},t)\in\Omega\times(0,T_{\mathrm{te}}],
\label{eq:pde_family_main}
\end{equation}
with initial and boundary constraints
\begin{equation}
u(\pmb{x},0;\pmb{\mu})=u_0(\pmb{x};\pmb{\mu}),
\qquad
\mathcal{B}_{\pmb{\mu}}[u](\pmb{x},t)=0,
\quad (\pmb{x},t)\in\partial\Omega\times(0,T_{\mathrm{te}}],
\label{eq:ic_bc_main}
\end{equation}
where $\mathcal{N}_{\pmb{\mu}}$ denotes the PDE operator, $\mathcal{B}_{\pmb{\mu}}$ denotes the boundary operator, and $u:\Omega\times[0,T_{\mathrm{te}}]\times\mathcal{P}_{\mathrm{eval}} \rightarrow\mathbb{R}^{d_u}$ is the unknown field.
Here $d_x$, $d_\mu$, and $d_u$ are the spatial, parameter, and output dimensions, respectively.
$u^*$ denotes the exact solution in the analysis; experiments use an exact or numerical reference as specified for each benchmark.
The present formulation assumes that the initial condition is fixed or determined by $\pmb{\mu}$.
Bold symbols denote vector-valued spatial coordinates and PDE parameters; the scalar symbol $x$ is reserved for explicitly one-dimensional formulations.
The goal is to learn a neural solution operator
\begin{equation}
\mathcal{G}_{\Theta}:(\pmb{x},t,\pmb{\mu})\mapsto
\hat u_{\Theta}(\pmb{x},t;\pmb{\mu}),
\end{equation}
where $\Theta$ denotes the trainable model parameters.
The operator approximates the solution $u^*(\pmb{x},t;\pmb{\mu})$ on observed samples and on held-out temporal intervals and unseen PDE parameters.

We explicitly separate parameter and time generalization regimes.
The training parameter set is a subset of the evaluation domain,
\begin{equation}
\mathcal{P}_{\mathrm{train}}\subset\mathcal{P}_{\mathrm{eval}},
\qquad
\mathcal{P}_{\mathrm{interp}}
\subset \operatorname{conv}(\mathcal{P}_{\mathrm{train}})
\setminus\mathcal{P}_{\mathrm{train}},
\qquad
\mathcal{P}_{\mathrm{extra}}
\subset \mathcal{P}_{\mathrm{eval}}
\setminus\operatorname{conv}(\mathcal{P}_{\mathrm{train}}),
\label{eq:param_split_main}
\end{equation}
where $\operatorname{conv}$ denotes the convex hull, $\mathcal{P}_{\mathrm{interp}}$ contains unseen parameters inside the training range, and $\mathcal{P}_{\mathrm{extra}}$ contains unseen parameters outside the training range.
Similarly, the time domain is split into the observed interval $[0,T_{\mathrm{tr}}]$ and the extrapolation interval $(T_{\mathrm{tr}},T_{\mathrm{te}}]$:
\begin{equation}
[0,T_{\mathrm{te}}]=[0,T_{\mathrm{tr}}]
\cup(T_{\mathrm{tr}},T_{\mathrm{te}}].
\label{eq:time_split_main}
\end{equation}

The parameter and time partitions define four evaluation regimes: parameter interpolation, parameter extrapolation, temporal extrapolation, and joint parameter--time extrapolation.
As a general model class, consider the scalar, $d_x$-dimensional parameterized Convection--Diffusion--Reaction (CDR) equation:
\begin{equation}
\label{eq:eq_cdr}
\frac{\partial u}{\partial t}
+ \beta\cdot\nabla u
- \nu \Delta u
- \rho u(1-u)=0,
\qquad \pmb{x}\in\Omega\subset\mathbb{R}^{d_x},\; t\in[0,T_{\mathrm{te}}],
\quad d_x\in\{1,2,3\},
\end{equation}
where $d_u=1$, $\beta\in\mathbb{R}^{d_x}$ is the convection vector, $\nu\ge0$ is the diffusion parameter (with zero denoting the absence of diffusion), and $\rho$ scales the Fisher-type reaction term $u(1-u)$.
The PDE parameter is $\pmb{\mu}=(\beta,\nu,\rho)\in\mathcal{P}_{\mathrm{eval}}$, with $d_\mu=d_x+2$.
For $d_x=1$, Eq.~\ref{eq:eq_cdr} reduces to the one-dimensional CDR equation with $\beta\cdot\nabla u=\beta \partial u/\partial x$ and $\Delta u=\partial^2 u/\partial x^2$.

PINNs face well-documented limitations in optimization and generalization for nonlinear parameterized PDEs.
A separate model is also commonly retrained for each parameter configuration.
Meta-PINNs mitigate the latter cost by learning a global solver $u_{\Theta}(\pmb{x},t;\pmb{\mu})$ over the parameter domain~\cite{liu2022novel}.
DLDMF combines an amortized parameter-to-latent initializer with parameter-conditioned latent evolution in one surrogate for parameter generalization and finite-horizon \textit{Out-t} prediction.

\subsection{Auto-decoding for latent space dynamics}
\label{sec:autodec_dyn}
A common approach to finite-horizon forecasting learns intrinsic latent dynamics on a low-dimensional manifold and decodes them back to the physical field \cite{yin2023dino,wang2025pido}.
Given discretized field snapshots $\{v_t\}_{t\in\mathcal{T}}$ at observation times $\mathcal T$, an encoder computes latent states $\alpha_t = E_\varphi(v_t)$, and a decoder reconstructs $v_t$ via $D_\phi(\alpha_t)$.
The latent evolution is then modeled continuously in time with a Neural Ordinary Differential Equation (Neural ODE):
\begin{equation}
\label{eq:latent_ode}
\frac{d\alpha_t}{dt} = f_\psi(\alpha_t), \qquad
\alpha_{t+\tau} = \alpha_t + \int_{0}^{\tau} f_\psi(\alpha_{t+s})\,ds,
\end{equation}
which permits continuous-time queries beyond the training horizon through numerical integration; predictive accuracy in that regime remains an empirical question.

A key challenge arises at deployment: only the initial condition $v_0$ is available, so the latent state $\alpha_0$ must be inferred to initialize the ODE.
In auto-decoding INR variants, latent codes are optimized without a feed-forward encoder~\citep{park2019deepsdf}.
With decoding loss $\ell_{\mathrm{dec}}(\phi,\alpha_t;v_t)=\|D_\phi(\alpha_t)-v_t\|_2^2$, auto-decoding defines the ``encoder'' as an iterative optimization procedure:
\begin{align}
\label{eq:autodec}
\textstyle E_{\mathrm{AD}}(v_t)=\alpha_t^{K}, \qquad
\alpha_t^{k+1}=\alpha_t^{k}-\eta \nabla_{\alpha_t}\ell_{\mathrm{dec}}(\phi,\alpha_t^{k};v_t),\quad k=0,\dots,K-1,
\end{align}
where $\alpha_t^0$ is an initial guess, $\eta$ is the step size, $K$ is the number of updates, and $\ell_{\mathrm{dec}}$ is evaluated on the observed grid.
Compared with amortized encoding, auto-decoding can improve instance-specific reconstruction and accommodates irregular observation grids without prescribing an encoder topology.

The additional reconstruction flexibility introduces an accuracy--latency trade-off.
Auto-decoding incurs iterative inference costs that grow with the number of query instances and may require storing parameter-specific latent codes and optimizer states~\citep{huang2022meta}.
Representative latent-dynamics pipelines also combine the procedure with Reduced-Order Modeling (ROM)~\cite{fresca2020deep,chen2023crom} or other low-dimensional representations~\cite{yin2023dino} before Neural ODE propagation.
DLDMF investigates a deterministic parameter-conditioned initializer that removes iterative test-time latent-code optimization while foregoing instance-specific latent refinement.

\subsection{Static parameterized solution manifolds}
\label{sec:static_manifold}
Parameterized PINN architectures address these limitations by learning a single model conditioned on PDE parameters, i.e., $u_{\Theta}(\pmb{x},t;\pmb{\mu})$~\cite{cho2024parameterized}.
P$^2$INN extends the approach by encoding PDE parameters into a hidden representation $h_{\mathrm{param}}=g_{\theta_p}(\pmb{\mu})$ that modulates the coordinate-based solver~\cite{cho2024parameterized}, expressed as:
\begin{equation}
\label{eq:p2inn_form}
h_{\mathrm{coord}}=g_{\theta_c}(\pmb{x},t),\quad
h_{\mathrm{param}}=g_{\theta_p}(\pmb{\mu}),\quad
u_{\Theta}(\pmb{x},t;\pmb{\mu})=g_{\theta_g}(h_{\mathrm{coord}},h_{\mathrm{param}}),
\end{equation}
where $g_{\theta_g}$ denotes the manifold network that integrates coordinate and parameter embeddings.

The formulation differs from instance-parameterized methods such as MAD, which optimize a separate latent vector per instance jointly with shared network weights~\cite{huang2022meta}.
By amortizing parameter representations globally through $g_{\theta_p}$, the parameterized representation paradigm avoids per-instance test-time optimization and shares parameter-dependent information across the training set.

Static parameterized solution manifolds treat $t$ as a coordinate without an explicit temporal evolution equation \cite{krishn2021char}.
The formulation lacks an intrinsic dynamical system, limiting \textit{Out-t} extrapolation beyond the training window \cite{yin2023dino,wang2025pido}.
Parameter extrapolation is similarly challenging when PDE dynamics change qualitatively across parameter ranges \cite{krishn2021char,lanth2022non}.
Explicit continuous-time latent dynamics extend the static parameterized manifold.

\subsection{Motivation}
\label{sec:motivation}

\paragraph{\textbf{Parameterized representations can replace auto-decoding.}}
Auto-decoding treats each instance as a separate optimization problem, where a latent code is iteratively updated, and meta-learning can organize instance adaptation \citep{park2019deepsdf,huang2022meta,liu2022novel}.
Instance-wise latent inference can be viewed as a multi-task learning procedure with task-specific latent variables \cite{kendall2018mti}.
In contrast, parameterized representation models encode collocation coordinates and PDE parameters via separate encoders and combine them in a shared decoder, producing parameter-conditioned representations directly without per-instance iterative optimization \cite{wang2021learn,krishn2021char}.

\paragraph{\textbf{Temporal limitations of static representations.}}
Although P$^2$INN~\cite{cho2024parameterized} constructs a parameterized solution manifold and generalizes well to unseen PDE parameters within the training time window, the model treats time $t$ as a static coordinate input rather than modeling intrinsic temporal dynamics.
As a result, when queried at times $t > T_{\mathrm{tr}}$, the model must extrapolate in time without an explicit dynamical structure, and its predictions can degrade as the temporal shift increases.
Fig.~\ref{fig:p2inn_extrapolation} illustrates the temporal extrapolation limitation under the evaluated setting.
At $t=1.0$, within the training horizon, the prediction follows the reference solution reasonably well.
At $t=5.0$ and $t=10.0$, the predicted solution exhibits substantial errors in both shape and magnitude.
The example shows that static parameterized decoding can be unreliable for \textit{Out-t} prediction and motivates explicit latent dynamics conditioned on PDE parameters.

\begin{figure}[H]
\centering
\includegraphics[width=.9\linewidth]{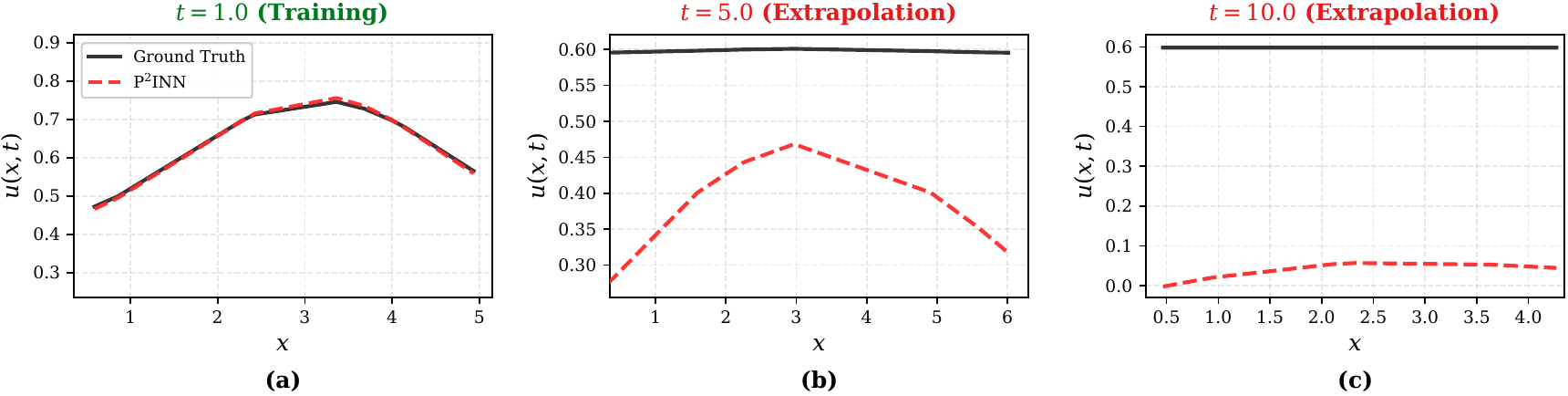}
\caption{P$^2$INN predictions and reference solutions for the one-dimensional CDR benchmark across the training and extrapolation horizons. (a) Training snapshot, $t=1.0$. (b) Extrapolation, $t=5.0$. (c) Extrapolation, $t=10.0$.}
\label{fig:p2inn_extrapolation}
\end{figure}

\begin{figure}[H]
\centering
\includegraphics[width=0.65\linewidth]{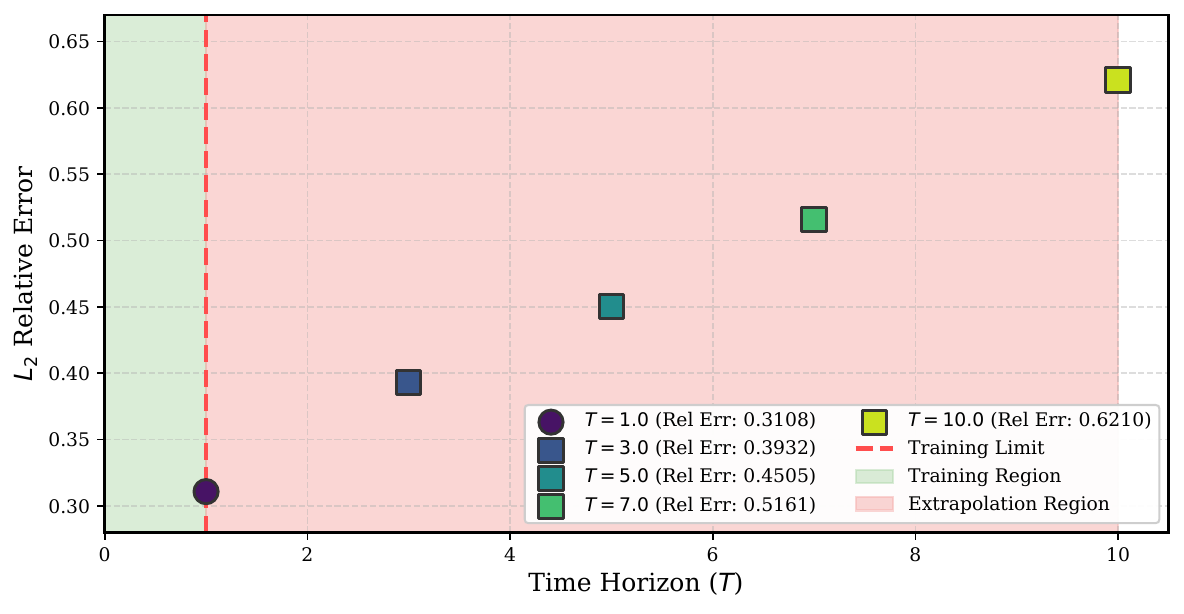}
\caption{Temporal growth of the P$^2$INN relative prediction error over time.}
\label{fig:p2inn_performance}
\end{figure}

\paragraph{\textbf{Latent dynamics on manifolds.}}
Solutions under nearby PDE parameters share common structures under smooth parameter dependence, suggesting that the corresponding evolution trajectories can lie on a low-dimensional manifold \cite{krishn2021char,huang2022meta}.
Latent dynamics models show that continuous-time evolution can be captured efficiently by learning a latent ODE~\cite{chen2018neural,chen2023crom} and decoding latent states to spatial fields~\cite{yin2023dino}.
The limitations of initial-condition-driven latent inference motivate a parameter-conditioned dynamics system that can be integrated continuously in time and reconstructed by a shared INR-style decoder \cite{wang2025pido,yin2023dino}.

As shown in Fig.~\ref{fig:p2inn_performance}, the relative $L_2$ error of P$^2$INN increases approximately linearly over the displayed extrapolation interval.
The observed error growth illustrates the difficulty of querying a static coordinate representation outside the training-time domain and further motivates explicit latent dynamics for finite-horizon prediction.

\paragraph{\textbf{Disentangled dynamic modulation.}}
Space--time separation strategies in INRs, such as multiplicative modulation mechanisms, simplify sequence modeling and improve efficiency \cite{fathony2020multi}.
Fourier-feature-based parameterizations can be interpreted as modulating amplitudes and phases of wave components, which helps represent oscillatory signals \cite{fathony2020multi,rahaman2019spectral}.
Directly factorizing time and space inside a single coordinate encoder may overlook the coupled nature of PDE dynamics and can lead to optimization difficulties in stiff regimes \cite{krishn2021char,wang2021under}.
The limitations of a single coordinate encoder motivate separate representations for space, latent time evolution~\cite{yin2023dino,wang2025pido}, and PDE parameters.
Combining the spatial, temporal, and parameter representations at the decoder permits cross-channel interactions while retaining the factorized architecture.


\section{Method}
\label{sec:meth}

We introduce DLDMF, a physics-informed latent time-integration surrogate for families of time-dependent parameterized PDEs.
The architecture uses separate spatial and parameter encoders and a parameter-conditioned latent Neural ODE.
A shared decoder combines the resulting spatial, parameter, and temporal representations.
Unlike latent-space dynamics pipelines that require per-instance auto-decoding to infer latent states (i.e., test-time iterative optimization), DLDMF initializes latent dynamics through a feed-forward mapping from PDE parameters, avoiding iterative latent optimization at deployment.

\subsection{Model Architecture}
\label{sec:arch}

DLDMF factorizes the input variables into three representation channels: a spatial code, a parameter code, and a parameter-conditioned temporal latent trajectory.
Let $d_h^x$, $d_h^\mu$, and $d_z$ denote the dimensions of the spatial, parameter, and temporal codes, respectively.
The component maps have signatures $g_{\theta_x}:\mathbb{R}^{d_x}\to\mathbb{R}^{d_h^x}$, $g_{\theta_p}:\mathbb{R}^{d_\mu}\to\mathbb{R}^{d_h^\mu}$, $g_{\theta_0}:\mathbb{R}^{d_h^\mu}\to\mathbb{R}^{d_z}$, and $f_{\theta_f}:\mathbb{R}^{d_z}\times\mathbb{R}^{d_h^\mu} \to\mathbb{R}^{d_z}$.

The spatial and parameter encoders are
\begin{equation}
\pmb{h}_{x}=g_{\theta_x}(\pmb{x}),
\qquad
\pmb{h}_{\mu}=g_{\theta_p}(\pmb{\mu}).
\label{eq:encoders_main}
\end{equation}

For fixed or parameter-determined initial conditions, the parameter embedding initializes the latent temporal state without auto-decoding as
\begin{equation}
\pmb{z}_0(\pmb{\mu})=g_{\theta_0}(\pmb{h}_{\mu}),
\label{eq:z0_init}
\end{equation}
and the latent dynamics evolve according to a parameter-conditioned Neural ODE as follows
\begin{equation}
\frac{d\pmb{z}_t}{dt}=f_{\theta_f}(\pmb{z}_t,\pmb{h}_{\mu}),
\qquad
\pmb{z}_t(\pmb{\mu})
=\phi_t^{\theta_f}\bigl(\pmb{z}_0(\pmb{\mu}),\pmb{h}_{\mu}\bigr),
\label{eq:latent_ode_dldmf}
\end{equation}
where $\phi_t^{\theta_f}$ denotes the flow map generated by $f_{\theta_f}$.
The finite-horizon continuous-time formulation permits temporal queries beyond the training horizon through latent-flow integration; DLDMF's \textit{Out-t} accuracy is evaluated empirically in Sec.~\ref{sec:exp}.

The implemented solution map concatenates the spatial, temporal, and parameter representations before applying the shared decoder.
With $d_m=d_h^x+d_z+d_h^\mu$, define the fixed map $\mathcal{F}_{\mathrm{cat}}: \mathbb{R}^{d_h^x}\times\mathbb{R}^{d_z}\times\mathbb{R}^{d_h^\mu} \to\mathbb{R}^{d_m}$ and $D_{\theta_d}=g_{\theta_g}:\mathbb{R}^{d_m}\to\mathbb{R}^{d_u}$, with $\theta_d\equiv\theta_g$.

The fused representation and decoded field are
\begin{equation}
\pmb{m}_{\Theta}(\pmb{x},t,\pmb{\mu})
=\mathcal{F}_{\mathrm{cat}}
(\pmb{h}_x,\pmb{z}_t(\pmb{\mu}),\pmb{h}_{\mu}),
\qquad
\hat u_{\Theta}(\pmb{x},t;\pmb{\mu})
=D_{\theta_d}\bigl(\pmb{m}_{\Theta}(\pmb{x},t,\pmb{\mu})\bigr).
\label{eq:fusion_decoder_main}
\end{equation}

Equivalently, the full DLDMF operator is
\begin{equation}
\hat u_{\Theta}(\pmb{x},t;\pmb{\mu})=
D_{\theta_d}\!\left(
\mathcal{F}_{\mathrm{cat}}\!\left(
g_{\theta_x}(\pmb{x}),
\phi_t^{\theta_f}\!\left(
g_{\theta_0}(g_{\theta_p}(\pmb{\mu})),
g_{\theta_p}(\pmb{\mu})
\right),
g_{\theta_p}(\pmb{\mu})
\right)
\right).
\label{eq:dldmf_full_operator_main}
\end{equation}

The fixed representation-level fusion map is
\begin{equation}
\mathcal{F}_{\mathrm{cat}}(\pmb{h}_x,\pmb{z}_t,\pmb{h}_{\mu})
=[\pmb{h}_x;\pmb{z}_t;\pmb{h}_{\mu}],
\qquad
D_{\theta_d}=g_{\theta_g},
\label{eq:concat_fusion_main}
\end{equation}
so that
\begin{equation}
\hat u_{\Theta}(\pmb{x},t;\pmb{\mu})
=g_{\theta_g}\Big([\pmb{h}_{x};\pmb{z}_t(\pmb{\mu});\pmb{h}_{\mu}]\Big).
\label{eq:decoder}
\end{equation}

Eq.~\ref{eq:concat_fusion_main} implements decoder-input concatenation of three disentangled representation streams corresponding to spatial information, parameter-dependent regime information, and latent temporal dynamics.
The decoder $g_{\theta_g}$ then learns nonlinear cross-channel interactions among $\pmb{h}_x$, $\pmb{z}_t$, and $\pmb{h}_{\mu}$.
And the full parameter set is $\Theta=\{\theta_x,\theta_p,\theta_0,\theta_f,\theta_g\}$.

DLDMF's fusion mechanism combines separately constructed spatial, temporal, and parameter representations in a shared nonlinear decoder.
More expressive operators based on gating, Feature-wise Linear Modulation (FiLM), bilinear maps, or attention are possible learnable extensions of $\mathcal{F}_{\mathrm{cat}}$.

\paragraph{\textbf{Time derivatives through latent dynamics.}}
For the strong-form residuals considered here, the spatial encoder and decoder are assumed to be sufficiently differentiable to evaluate the required spatial derivatives.
Since $t$ influences $\hat u_{\Theta}$ only through $\pmb{z}_t$, the time derivative required by PDE residuals follows from the chain rule.

Let $J_zD_{\theta_d}(\pmb{m}_{\Theta})\in\mathbb{R}^{d_u\times d_z}$ denote the decoder Jacobian with respect to the latent temporal block of the fused input.
Then
\begin{align}
\partial_t \hat u_{\Theta}(\pmb{x},t;\pmb{\mu})
&=
J_zD_{\theta_d}(\pmb{m}_{\Theta})
f_{\theta_f}(\pmb{z}_t,\pmb{h}_{\mu}).
\label{eq:time_derivative}
\end{align}

Spatial derivatives (e.g., $\nabla_{\pmb{x}}\hat u_{\Theta}$ and
$\Delta_{\pmb{x}}\hat u_{\Theta}$) are obtained by automatic differentiation
through $\pmb{h}_x=g_{\theta_x}(\pmb{x})$ and the decoder.
Since the spatial coordinates enter the solution map only through
$\pmb{h}_x$, the required derivatives follow directly from the chain rule.
Higher-order spatial derivatives appearing in the PDE operator are evaluated
through the same computational graph.

\subsection{Pre-training}
\label{sec:train}

DLDMF is trained with a physics-informed objective whose terms correspond to data fitting, residual control, initial/boundary consistency, and latent-flow stability.
Let $\mathcal{D}_u$, $\mathcal{D}_r$, $\mathcal{D}_{0}$, and $\mathcal{D}_b$ denote supervised, residual, initial-condition, and boundary-condition sample sets.
The pointwise normalized squared supervised loss uses a denominator regularizer $\epsden>0$, which has the same units as the solution and prevents division by zero when a reference value has a small norm.

The numerical value of $\epsden$ is fixed across comparisons and recorded in the benchmark:
\begin{equation}
\mathcal{L}_{u}(\Theta)=
\frac{1}{|\mathcal{D}_{u}|}
\sum_{(\pmb{x}_i,t_i,\pmb{\mu}_i,u_i)\in\mathcal{D}_{u}}
\frac{\|\hat u_{\Theta}(\pmb{x}_i,t_i;\pmb{\mu}_i)-u_i\|_2^2}
{\|u_i\|_2^2+\epsden^2}.
\label{eq:supervised_loss_main}
\end{equation}

The physics residual, with $\partial_t\hat u_{\Theta}$ computed by Eq.~\ref{eq:time_derivative}, is penalized by
\begin{equation}
\mathcal{L}_{r}(\Theta)=
\frac{1}{|\mathcal{D}_{r}|}
\sum_{(\pmb{x}_i,t_i,\pmb{\mu}_i)\in\mathcal{D}_{r}}
\left\|
\mathcal{N}_{\pmb{\mu}_i}[\hat u_{\Theta}](\pmb{x}_i,t_i)
\right\|_2^2.
\label{eq:residual_loss_main}
\end{equation}

The initial and boundary conditions are enforced through
\begin{align}
\mathcal{L}_{0}(\Theta)&=
\frac{1}{|\mathcal{D}_{0}|}
\sum_{(\pmb{x}_i,\pmb{\mu}_i)\in\mathcal{D}_{0}}
\|\hat u_{\Theta}(\pmb{x}_i,0;\pmb{\mu}_i)-u_0(\pmb{x}_i;\pmb{\mu}_i)\|_2^2,
\label{eq:ic_loss_main} \\
\mathcal{L}_{b}(\Theta)&=
\frac{1}{|\mathcal{D}_{b}|}
\sum_{(\pmb{x}_i,t_i,\pmb{\mu}_i)\in\mathcal{D}_{b}}
\|\mathcal{B}_{\pmb{\mu}_i}[\hat u_{\Theta}](\pmb{x}_i,t_i)\|_2^2.
\label{eq:bc_loss_main}
\end{align}

The theoretical analysis relates finite-horizon perturbation growth to the one-sided stability of the latent vector field.
Let $\mathcal{D}_z$ denote the set of parameter--time samples selected for latent regularization, and let $\|\cdot\|_F$ denote the Frobenius norm.

Then, the Jacobian of the learned vector field is computed as
\begin{equation}
\mathcal{R}_{J}(\Theta)=
\frac{1}{|\mathcal{D}_{z}|}
\sum_{(t_i,\pmb{\mu}_i)\in\mathcal{D}_{z}}
\left\|
\frac{\partial f_{\theta_f}}{\partial \pmb{z}}
\bigl(\pmb{z}_{t_i}(\pmb{\mu}_i),\pmb{h}_{\mu}(\pmb{\mu}_i)\bigr)
\right\|_{F}^{2}.
\label{eq:jacobian_regularizer_main}
\end{equation}

The regularizer provides an average squared Frobenius-norm penalty for sampled local sensitivity of the learned vector field.
The uniform one-sided Lipschitz constant $\alpha^*$ of the ideal vector field is defined separately in Eq.~\ref{eq:main_one_sided_lipschitz}.
Fig.~\ref{fig:dldmf_complex_dyn} associates the regularizer with lower latent oscillation and fewer localized error peaks at intermediate times.
The observed behavior is consistent with smoother latent trajectories and more stable spatiotemporal predictions over the displayed horizon.

\begin{figure}[H]
\centering
\includegraphics[width=0.85\linewidth]{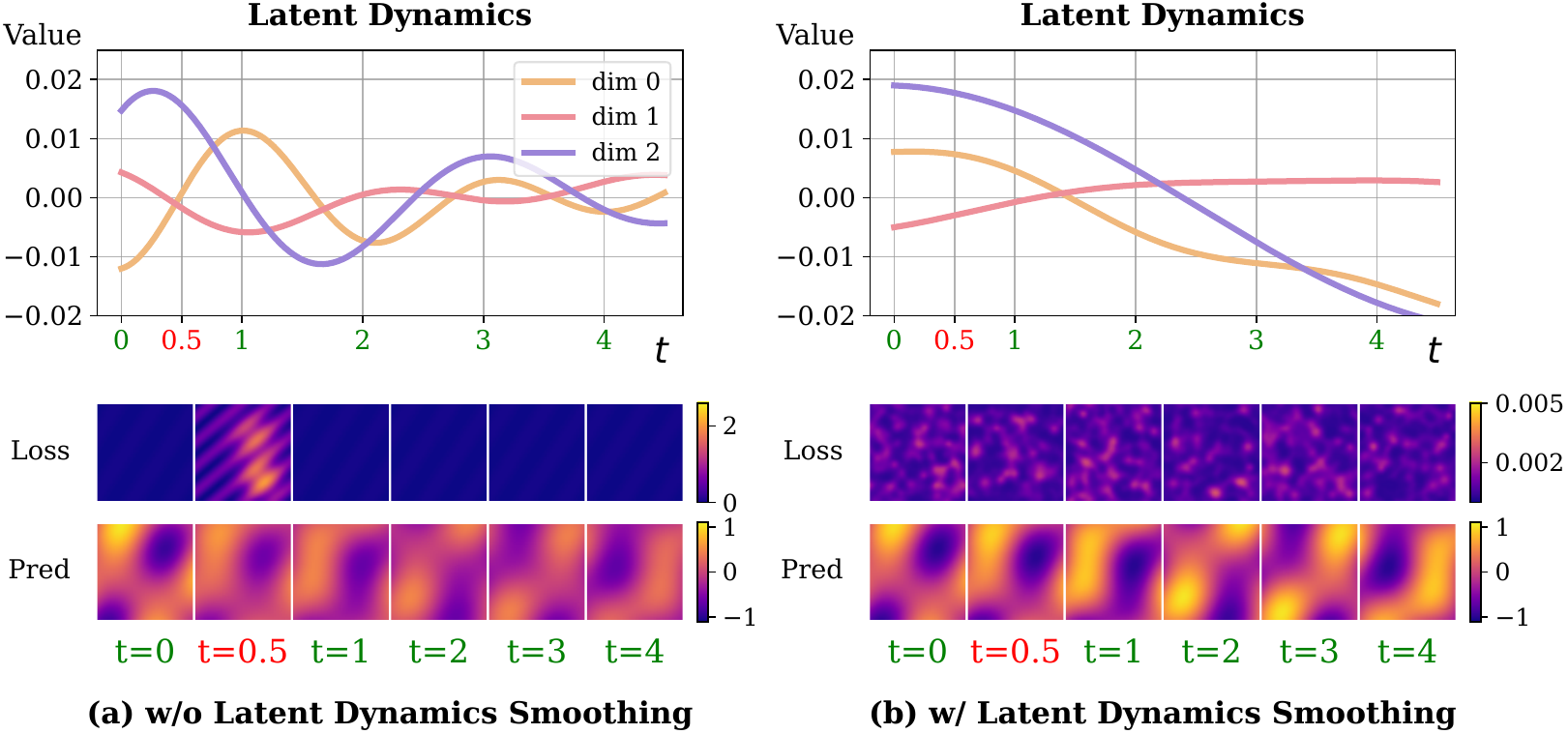}
\caption{Latent trajectories and errors in the local-sensitivity comparison. (a) Without smoothing. (b) With smoothing.}
\label{fig:dldmf_complex_dyn}
\end{figure}

The full pre-training objective is
\begin{equation}
\min_{\Theta}\;\mathcal{J}(\Theta)
=\lambda_u\mathcal{L}_{u}
+\lambda_r\mathcal{L}_{r}
+\lambda_0\mathcal{L}_{0}
+\lambda_b\mathcal{L}_{b}
+\lambda_J\mathcal{R}_{J}
+\lambda_{\mathrm{wd}}\|\Theta\|_2^2,
\label{eq:loss_dldmf}
\end{equation}
in which $\lambda_u$, $\lambda_r$, $\lambda_0$, $\lambda_b$, $\lambda_J$, and $\lambda_{\mathrm{wd}}$ are nonnegative loss weights.
Here, $\|\Theta\|_2^2$ denotes the sum of squared entries over all trainable parameter blocks.
When $\mathcal{R}_J$ is disabled, we set $\lambda_J=0$ and treat the corresponding variant as a stability ablation.
Compared with PINNs trained for a single PDE instance, each optimization step mini-batches across multiple parameter instances, encouraging a parameter-conditioned solution manifold.

\subsection{Decoder modulation for adaptation}
\label{sec:decoder_modulation}
The core DLDMF inference pipeline uses the pretrained parameter encoder, latent initializer, latent ODE, and decoder for direct prediction.
We additionally include an optional decoder-only adaptation protocol to compare different fine-tuning decompositions for selected target PDE parameters.

For a target parameter $\pmb{\mu}_{\mathrm{tgt}}$, let $\vartheta$ denote the trainable variables of a selected modulation variant.
All adaptive variants solve the common adaptation problem
\begin{equation}
\vartheta^*=\arg\min_{\vartheta}
\mathcal{J}_{\mathrm{adapt}}(\vartheta;\pmb{\mu}_{\mathrm{tgt}}),
\label{eq:decoder_adaptation_objective}
\end{equation}
where $\mathcal{J}_{\mathrm{adapt}}$ is the restriction of the weighted objective in Eq.~\ref{eq:loss_dldmf} to the target-parameter adaptation samples.
The variants differ only in the definition and dimension of $\vartheta$.

Following the modulated-INR viewpoint~\cite{functa22functa}, the adaptation variants can be grouped by where the learnable correction is inserted.
Let $q_l=W_l h_{l-1}+b_l\in\mathbb{R}^{n_l}$ denote the pre-activation of decoder layer $l$, where $h_{l-1}\in\mathbb R^{n_{l-1}}$ and $b_l\in\mathbb R^{n_l}$ are the input and bias of layer $l$, respectively.
Activation-space variants modify $q_l$ directly, e.g.,
\begin{align}
\text{Shift:}\quad &q_l \mapsto q_l+a_l,\notag\\
\text{Scale:}\quad &q_l \mapsto q_l\odot a_l,\notag\\
\text{FiLM:}\quad &q_l \mapsto q_l\odot a_l^{(s)}+a_l^{(b)},
\label{eq:decoder_activation_modulation}
\end{align}
where $a_l,a_l^{(s)},a_l^{(b)}\in\mathbb{R}^{n_l}$, with elementwise multiplication denoted by $\odot$.
The Shift, Scale, and FiLM variants keep the pretrained decoder weights fixed and adapt only the layer-wise affine response.
In contrast, weight-space variants modify the linear map itself.
Global Fourier Modulation (GFM) changes the decoder through a basis-weighted modulation of spectral components, whereas the Singular Value Decomposition (SVD) variant modulates singular directions extracted from the pretrained decoder weights.
Full fine-tuning updates all decoder parameters, while the fixed-pretrained variant retains the pretrained decoder parameters.

Let the decoder contain $D_g$ layers.
For the SVD variant, the compact SVD of each decoder weight matrix, excluding the first and last layers, is
\begin{align}
W_l=U_l\operatorname{diag}(\pmb{\sigma}_l)V_l^{\mathsf{T}},
\qquad l=2,3,\ldots,D_g-1,
\label{eq:eq_svd_mod}
\end{align}
where $r_l=\operatorname{rank}(W_l)$, $U_l\in\mathbb{R}^{n_l\times r_l}$, $V_l\in\mathbb{R}^{n_{l-1}\times r_l}$, and $\pmb{\sigma}_l\in\mathbb{R}^{r_l}$ contains the singular values.
During fine-tuning, the adapted weight is
\begin{equation}
W_l^{\mathrm{adapt}}
=U_l\operatorname{diag}(\pmb{\sigma}_l\odot\pmb{a}_l)V_l^{\mathsf{T}},
\qquad \pmb{a}_l\in\mathbb{R}^{r_l},
\label{eq:eq_svd_adapt}
\end{equation}
where $\pmb{a}_l$ is a vector of directional gains initialized to the all-ones vector.
The factors $U_l$ and $V_l$, the original singular values, the first and last decoder layers, and the remaining network parameters are fixed in the SVD variant.
The construction adapts the gain along pretrained singular directions rather than changing all entries of $W_l$ independently.

The SVD construction provides a parameter-efficient weight-space counterpart to Shift, Scale, FiLM, GFM, full decoder fine-tuning, and fixed-pretrained evaluation.
The core DLDMF pipeline consists of feed-forward parameter-to-latent initialization, latent ODE evolution, and reconstruction with the pretrained decoder.

\subsection{Extrapolation}
\label{sec:extra}

For a new PDE parameter vector $\pmb{\mu}_{\mathrm{new}}$, potentially outside the training range, DLDMF computes
\begin{align}
\pmb{h}_{\mu,\mathrm{new}} &= g_{\theta_p}(\pmb{\mu}_{\mathrm{new}}), \notag\\
\pmb{z}_{0,\mathrm{new}} &= g_{\theta_0}(\pmb{h}_{\mu,\mathrm{new}}), \notag\\
\pmb{z}_{t,\mathrm{new}} &= \mathrm{ODESolve}\!\left(
f_{\theta_f}(\,\cdot\,,\pmb{h}_{\mu,\mathrm{new}}),
\pmb{z}_{0,\mathrm{new}},0,t\right).
\end{align}

The latent state $\pmb{z}_{t,\mathrm{new}}$ is evolved by Eq.~\ref{eq:latent_ode_dldmf} for any $t \in [0,T_{\mathrm{te}}]$, including $t > T_{\mathrm{tr}}$, and $\hat u_{\Theta}(\pmb{x},t;\pmb{\mu}_{\mathrm{new}})$ is decoded by Eq.~\ref{eq:decoder}.
The procedure involves only feed-forward computation and ODE integration, permitting parameter and temporal queries without instance-wise latent optimization.
The extrapolation accuracy depends on the learned parameter map and latent vector field and is assessed in Sec.~\ref{sec:exp}.

\subsection{Theoretical properties}
\label{sec:theory_main}

The analysis provides a conditional finite-horizon decomposition of temporal, reconstruction-interface, and numerical errors under the stated uniform approximation and regularity assumptions, characterizing the propagation of approximation and dynamical errors over the selected horizon.
Out-of-support prediction is assessed empirically in Sec.~\ref{sec:exp}.
The Jacobian regularizer in Eq.~\ref{eq:jacobian_regularizer_main} is interpreted as a sampled empirical surrogate for local latent sensitivity.

Fix a horizon $0<T\le T_{\mathrm{te}}$ and the Euclidean norm on each latent space, together with the induced operator norm.
On product spaces, define $\|(a,b,c)\|_{\oplus}=\|a\|+\|b\|+\|c\|$.
To reduce typographic clutter, boldface is suppressed for latent vectors throughout this subsection.
Let $h_x^*:\Omega\to\mathbb{R}^{d_h^x}$, $h_\mu^*:\mathcal{P}_{\mathrm{eval}}\to\mathbb{R}^{d_h^\mu}$, and $z^*:[0,T]\times\mathcal{P}_{\mathrm{eval}}\to\mathbb{R}^{d_z}$ denote ideal representation maps.
Define a reconstruction operator $H^*:\mathbb{R}^{d_h^x}\times\mathbb{R}^{d_z}\times \mathbb{R}^{d_h^\mu}\to\mathbb{R}^{d_u}$ and the learned reconstruction $H_\Theta=D_{\theta_d}\circ\mathcal{F}_{\mathrm{cat}}$.
The corresponding finite-dimensional representation is
\begin{equation}
u^*(\pmb{x},t;\pmb{\mu})
=H^*\!\left(h_x^*(\pmb{x}),z_t^*(\pmb{\mu}),
h_\mu^*(\pmb{\mu})\right)+\eta_{\mathrm{man}}(\pmb{x},t,\pmb{\mu}),
\qquad
\sup_{\pmb{x},t,\pmb{\mu}}\norm{\eta_{\mathrm{man}}(\pmb{x},t,\pmb{\mu})}
\le\eps_{\mathrm{man}}.
\label{eq:main_ideal_representation}
\end{equation}

We define the truncation term with respect to the prescribed DLDMF-compatible representation class.
Let $\mathfrak{A}_{d_h^x,d_z,d_h^\mu}$ denote the admissible class of continuous factorizations associated with the prescribed spatial encoder, latent-dynamics, parameter-encoder, and decoder function classes at the specified representation dimensions, and define the nonlinear approximation width
\begin{equation}
d_{\mathfrak A}(\mathcal{M}_T)=
\inf_{(h_x,h_\mu,z,H)\in\mathfrak A}
\sup_{\pmb{x},t,\pmb{\mu}}
\norm{u^*(\pmb{x},t;\pmb{\mu})
-H(h_x(\pmb{x}),z_t(\pmb{\mu}),h_\mu(\pmb{\mu}))},
\label{eq:main_nonlinear_width}
\end{equation}
where $\mathcal{M}_T$ is the finite-horizon solution family.
The infimum is restricted to the prescribed DLDMF-compatible factorization class.
The quantity $\eps_{\mathrm{man}}$ is taken as an attained or approximate upper bound on $d_{\mathfrak A}(\mathcal{M}_T)$ for the selected dimensions.

The uniform approximation defects are
\begin{align}
\norm{g_{\theta_x}(\pmb{x})-h_x^*(\pmb{x})}&\le\eps_x,\qquad
\norm{g_{\theta_p}(\pmb{\mu})-h_\mu^*(\pmb{\mu})}\le\eps_\mu,
\notag\\
\norm{z_t(\pmb{\mu})-z_t^*(\pmb{\mu})}&\le\eps_z(t),\qquad
\sup_a\norm{H_\Theta(a)-H^*(a)}\le\eps_{\mathrm{rec}}.
\label{eq:main_approx_defects}
\end{align}

The constants $L_H$, $L_z$, and $L_\mu$ denote finite-horizon Lipschitz constants of the learned reconstruction, latent vector field in $z$, and latent vector field in the parameter embedding, respectively.
One may take $L_H\le L_DL_F$, where $L_D$ and $L_F$ are the decoder and fusion-map Lipschitz constants.
The suprema are taken over the compact domains reached by the encoders and latent trajectories.
The combined defect $\eps_{\mathrm{rec}}$ represents the composition-level approximation error of the fusion interface and decoder.
Invertible changes of intermediate coordinates preserve the fusion--decoder composition while redistributing module-wise approximation error.

\begin{proposition}[Finite-horizon continuity of DLDMF]
\label{prop:main_continuity}
Assume that $g_{\theta_x}$, $g_{\theta_p}$, $g_{\theta_0}$, $\mathcal{F}_{\mathrm{cat}}$, and $D_{\theta_d}$ are continuous and that $f_{\theta_f}(z,h)$ is jointly continuous and locally Lipschitz in $z$, uniformly over the compact parameter-embedding domain, with finite-horizon solutions on the considered compact set.
Then $(\pmb{x},t,\pmb{\mu})\mapsto\hat u_{\Theta}(\pmb{x},t;\pmb{\mu})$ is continuous on $\Omega\times[0,T]\times\mathcal{P}_{\mathrm{eval}}$.
\end{proposition}

\begin{corollary}[Distance-dependent prediction error bound]
\label{cor:main_parameter_distance}
Assume that, at a fixed $t\in[0,T]$, the learned predictor and exact solution, viewed as maps into a fixed normed output space $X$, are Lipschitz in $\pmb{\mu}$ with constants $K_\Theta(t)$ and $K_*(t)$, respectively.
For any query $\pmb{\mu}$ and any reference parameter $\pmb{\mu}_0\in\mathcal{P}_{\mathrm{train}}$,
\begin{align}
\norm{\hat u_\Theta(\cdot,t;\pmb{\mu})-u^*(\cdot,t;\pmb{\mu})}_X
\le
\norm{\hat u_\Theta(\cdot,t;\pmb{\mu}_0)
-u^*(\cdot,t;\pmb{\mu}_0)}_X \notag +\bigl(K_\Theta(t)+K_*(t)\bigr)
\norm{\pmb{\mu}-\pmb{\mu}_0}.
\label{eq:main_parameter_distance_bound}
\end{align}

Choosing a nearest training parameter makes the second term proportional to the distance from the observed parameter set.
The bound links prediction error to the distance from the observed parameter set through the parameter-sensitivity constants.
\end{corollary}

\paragraph{\textbf{Latent-flow error propagation and temporal extrapolation.}}
Further assume that the ideal trajectory $z_t^*$ is generated by an autonomous vector field $f^*:\mathbb{R}^{d_z}\times\mathbb{R}^{d_h^\mu} \to\mathbb{R}^{d_z}$, and let $z_t$ denote the DLDMF trajectory:
\begin{equation}
\dot z_t^*=f^*(z_t^*,h_\mu^*),\qquad
z_0^*=z^*(0;\pmb{\mu}),\qquad
\dot z_t=f_{\theta_f}(z_t,h_\mu),\qquad
z_0=g_{\theta_0}(h_\mu).
\label{eq:main_ideal_and_learned_latent_flow}
\end{equation}

The following conditional bound propagates latent initialization, vector-field approximation, and parameter-embedding defects over the selected finite horizon.

\begin{proposition}[Temporal latent perturbation]
\label{prop:main_temporal_perturbation}
Assume that $f^*$ is $\alpha^*$-one-sided Lipschitz in $z$ and $L_\mu^*$-Lipschitz in $h_\mu$, namely
\begin{align}
\inner{z_1-z_2}{f^*(z_1,h)-f^*(z_2,h)}
&\le \alpha^*\norm{z_1-z_2}^2,
\notag\\
\norm{f^*(z,h_1)-f^*(z,h_2)}
&\le L_\mu^*\norm{h_1-h_2}.
\label{eq:main_one_sided_lipschitz}
\end{align}

Suppose
\begin{equation}
\norm{z_0-z_0^*}\le\eps_0,\qquad
\norm{h_\mu-h_\mu^*}\le\eps_\mu,\qquad
\sup_{z,h}\norm{f_{\theta_f}(z,h)-f^*(z,h)}\le\eps_f.
\label{eq:main_latent_defects}
\end{equation}

Then, for every $t\in[0,T]$,
\begin{equation}
\norm{z_t-z_t^*}
\le
e^{\alpha^* t}\eps_0
+\chi_{\alpha^*}(t)
(\eps_f+L_\mu^*\eps_\mu),
\label{eq:main_temporal_bound}
\end{equation}
where $\chi_\alpha(t)=(e^{\alpha t}-1)/\alpha$ for $\alpha\ne0$ and $\chi_0(t)=t$.
If $f^*$ is $L_z^*$-Lipschitz in $z$, one may take $\alpha^*=L_z^*$; $\alpha^*<0$ identifies a contractive latent flow.
\end{proposition}

\begin{proof}
Let $e_z(t)=\norm{z_t-z_t^*}$ and $D^+e_z(t)=\limsup_{\tau\downarrow0}[e_z(t+\tau)-e_z(t)]/\tau$.
Applying Eq.~\ref{eq:main_one_sided_lipschitz} to the state-dependent term and incorporating the vector-field and embedding defects gives the upper-Dini inequality
\begin{align}
D^+e_z(t)
&\le \alpha^* e_z(t)+\eps_f+L_\mu^*\eps_\mu .
\label{eq:main_temporal_dini}
\end{align}

Applying Gr\"onwall's inequality with $e_z(0)\le\eps_0$ gives Eq.~\ref{eq:main_temporal_bound}.
\end{proof}

\begin{corollary}[Finite-horizon output error induced by latent flow]
\label{cor:main_temporal_output}
If $H_\Theta$ is $L_H$-Lipschitz with respect to the latent variable, then
\begin{equation}
\norm{H_\Theta(h_x,z_t,h_\mu)-H_\Theta(h_x,z_t^*,h_\mu)}
\le L_H\norm{z_t-z_t^*}.
\label{eq:main_temporal_output_bound}
\end{equation}

Combining Eqs.~\ref{eq:main_temporal_bound} and \ref{eq:main_temporal_output_bound} gives a finite-horizon output error bound; the factor $e^{\alpha^* t}$ identifies the finite-horizon sensitivity to the one-sided stability constant of the ideal latent vector field.
\end{corollary}

\paragraph{\textbf{Output-error decomposition.}}
The next result separates encoding and latent-state errors, combined reconstruction error, and representation truncation error.

\begin{proposition}[Output-error decomposition]
\label{prop:main_fusion_decomp}
Under Eqs.~\ref{eq:main_ideal_representation}--\ref{eq:main_approx_defects}, for every $(\pmb{x},t,\pmb{\mu})\in\Omega\times[0,T]\times\mathcal{P}_{\mathrm{eval}}$,
\begin{equation}
\begin{aligned}
\norm{\hat u_{\Theta}(\pmb{x},t;\pmb{\mu})-u^*(\pmb{x},t;\pmb{\mu})}
&\le L_H(\eps_x+\eps_z(t)+\eps_\mu) \\
&\quad+\eps_{\mathrm{rec}}+\eps_{\mathrm{man}}.
\end{aligned}
\label{eq:main_fusion_decomp_bound}
\end{equation}
\end{proposition}

\begin{proof}
Write
\begin{equation}
a_\Theta=(g_{\theta_x}(\pmb{x}),z_t(\pmb{\mu}),g_{\theta_p}(\pmb{\mu})),
\qquad
a^*=(h_x^*(\pmb{x}),z_t^*(\pmb{\mu}),h_\mu^*(\pmb{\mu})).
\end{equation}

Applying Eq.~\ref{eq:main_ideal_representation} and adding and subtracting $H_\Theta(a^*)$,
\begin{align}
\norm{\hat u_\Theta-u^*}
&\le
\norm{H_\Theta(a_\Theta)-H_\Theta(a^*)}
+\norm{H_\Theta(a^*)-H^*(a^*)}
+\eps_{\mathrm{man}} \notag\\
&\le L_H\norm{a_\Theta-a^*}_{\oplus}
+\eps_{\mathrm{rec}}+\eps_{\mathrm{man}} \notag\\
&\le L_H(\eps_x+\eps_z(t)+\eps_\mu)
+\eps_{\mathrm{rec}}+\eps_{\mathrm{man}}.
\end{align}
\end{proof}

\begin{corollary}[Combined finite-horizon solution error]
\label{cor:main_combined_solution_bound}
Under Proposition~\ref{prop:main_temporal_perturbation} and Proposition~\ref{prop:main_fusion_decomp},
\begin{align}
\norm{\hat u_\Theta(\pmb{x},t;\pmb{\mu})-u^*(\pmb{x},t;\pmb{\mu})}
&\le L_H\left(
\eps_x+
e^{\alpha^* t}\eps_0
+\chi_{\alpha^*}(t)
(\eps_f+L_\mu^*\eps_\mu)
+\eps_\mu
\right) \notag \\
&\quad+\eps_{\mathrm{rec}}+\eps_{\mathrm{man}}.
\label{eq:main_combined_solution_bound}
\end{align}

The bound separates spatial encoding, parameter encoding, latent initialization, latent vector-field approximation, combined reconstruction approximation, and manifold truncation.
\end{corollary}

\paragraph{\textbf{Scope of statistical and optimization statements.}}
A standard Rademacher bound for a bounded scalar supervised loss concerns only the sampling distribution that generates the observed space--time--parameter points.
Moreover, a physics residual depends on derivatives of the hypothesis, so extending such a result to the residual loss requires a derivative-aware, for example Sobolev, function class.
The main analysis focuses on finite-horizon error propagation, numerical integration, and latent-sensitivity properties specific to DLDMF.

\paragraph{\textbf{Numerical integration and latent sensitivity.}}
The remaining statements concern numerical and variational properties specific to the parameter-conditioned latent ODE.

\begin{proposition}[Numerical latent-flow convergence]
\label{prop:main_ode_solver_convergence}
Assume that $f_{\theta_f}$ is sufficiently smooth on a neighborhood of the exact latent trajectory over $[0,T]$.
Let $z(t_n)$ be the exact latent state at the numerical nodes $\{t_n\}_{n=0}^{N_T}$, and let $\tilde z_n$ be the solution from a stable one-step method of order $p$ with maximal step size $\Delta t$ and $\tilde z_0=z(0)$.
Define the reconstruction map $\Psi_\Theta(\pmb{x},z,\pmb{\mu})= D_{\theta_d}(\mathcal{F}_{\mathrm{cat}}(g_{\theta_x}(\pmb{x}),z, g_{\theta_p}(\pmb{\mu})))$.
Then, for sufficiently small $\Delta t$, there exists $C_T>0$, independent of $\Delta t$, such that
\begin{equation}
\left\{
\begin{aligned}
\max_{0\le n\le N_T}\norm{\tilde z_n-z(t_n)}
&\le C_T(\Delta t)^p,\\
\norm{\Psi_\Theta(\pmb{x},\tilde z_n,\pmb{\mu})
-\Psi_\Theta(\pmb{x},z(t_n),\pmb{\mu})}
&\le L_D L_F C_T(\Delta t)^p.
\end{aligned}
\right.
\label{eq:main_ode_convergence_bound}
\end{equation}
\end{proposition}

\begin{proof}
The first inequality is the global error estimate for a stable $p$th-order one-step method applied to a smooth finite-horizon ODE at the numerical nodes.
The second follows by the $L_D L_F$-Lipschitz continuity of $D_{\theta_d}\circ\mathcal{F}_{\mathrm{cat}}$ in the latent argument.
\end{proof}

\begin{proposition}[Latent ODE sensitivity]
\label{prop:main_ode_sensitivity}
Let $S_t=\partial z_t/\partial\theta_f$ and $J_t=\partial z_t/\partial h_\mu$.
If $f_{\theta_f}$ is differentiable in $z$, $h_\mu$, and $\theta_f$, and $g_{\theta_0}$ is differentiable in $h_\mu$, then
\begin{align}
\frac{\dd S_t}{\dd t}
&=\frac{\partial f_{\theta_f}}{\partial z}(z_t,h_\mu)S_t
+\frac{\partial f_{\theta_f}}{\partial\theta_f}(z_t,h_\mu),
\qquad S_0=0, \notag\\
\frac{\dd J_t}{\dd t}
&=\frac{\partial f_{\theta_f}}{\partial z}(z_t,h_\mu)J_t
+\frac{\partial f_{\theta_f}}{\partial h_\mu}(z_t,h_\mu),
\qquad
J_0=\frac{\partial g_{\theta_0}}{\partial h_\mu}(h_\mu).
\label{eq:main_sensitivity_equations}
\end{align}

If $\norm{\partial f_{\theta_f}/\partial z}\le L_z$, $\norm{\partial f_{\theta_f}/\partial\theta_f}\le B_f$, $\norm{\partial f_{\theta_f}/\partial h_\mu}\le B_h$, and $\norm{\partial g_{\theta_0}/\partial h_\mu}\le B_0$ on $[0,T]$, then
\begin{equation}
\norm{S_t}\le\frac{e^{L_z t}-1}{L_z}B_f,\qquad
\norm{J_t}\le e^{L_z t}B_0+\frac{e^{L_z t}-1}{L_z}B_h,
\label{eq:main_sensitivity_bounds}
\end{equation}
with the convention that $(e^{L_z t}-1)/L_z=t$ when $L_z=0$.
\end{proposition}

\begin{proof}
Differentiating the latent ODE with respect to $\theta_f$ and $h_\mu$ gives Eq.~\ref{eq:main_sensitivity_equations}.
For $S_t$, the variation-of-constants formula gives
\begin{equation}
S_t=\int_0^t\Phi(t,s)
\frac{\partial f_{\theta_f}}{\partial\theta_f}(z_s,h_\mu)\,\dd s,
\end{equation}
where $\Phi(t,s)$ is the state-transition matrix associated with $\partial f_{\theta_f}/\partial z$.
Since $\norm{\Phi(t,s)}\le e^{L_z(t-s)}$,
\begin{equation}
\norm{S_t}\le\int_0^t e^{L_z(t-s)}B_f\,\dd s
=\frac{e^{L_z t}-1}{L_z}B_f.
\end{equation}

The bound for $J_t$ follows from the same argument with the nonzero initial sensitivity $J_0$, yielding the second inequality in Eq.~\ref{eq:main_sensitivity_bounds}.
\end{proof}

\section{Experiments}
\label{sec:exp}

We evaluate DLDMF on parameterized PDE benchmarks comprising time-dependent CDR and Navier--Stokes systems and a stationary Helmholtz problem.
The evolution benchmarks are assessed along parameter and temporal axes, whereas the Helmholtz benchmark isolates parameter generalization.
Following Sec.~\ref{sec:ppde}, the evaluation distinguishes interpolation, finite-horizon temporal extrapolation, parameter extrapolation, and joint parameter--time shifts.

\subsection{Benchmarks}
\label{sec:bench}

\paragraph{\textbf{Parameterized CDR equations.}}
The general CDR family is defined in Eq.~\ref{eq:eq_cdr}.
The principal quantitative comparisons adopt the $d_x=1$ specialization of the CDR family, with scalar convection, diffusion, and reaction parameters $(\beta,\nu,\rho)$, while higher-dimensional realizations with $d_x\in\{2,3\}$ provide qualitative field comparisons.
The initial and boundary conditions, spatial domain, parameter split, and reference discretization are fixed within each comparison.
The constituent operators arise in fluid mechanics, chemical transport, and biological systems, and the coupled CDR system spans transport-, diffusion-, and reaction-dominated dynamics~\cite{krishn2021char}.
The CDR family tests parameter generalization and finite-horizon temporal extrapolation within a common solution family.
Reference solutions are generated with the Strang operator-splitting method~\cite{strang1968cons}.

\paragraph{\textbf{Parameterized Navier--Stokes equations.}}
We consider the two-dimensional, nondimensional, incompressible Navier--Stokes system
\begin{equation}
\partial_t\pmb{v}+(\pmb{v}\cdot\nabla)\pmb{v}+\nabla p
-\mathrm{Re}^{-1}\Delta\pmb{v}=\pmb{f},
\qquad \nabla\cdot\pmb{v}=0,
\quad (\pmb{x},t)\in\Omega\times[0,T_{\mathrm{te}}],
\label{eq:ns_benchmark}
\end{equation}
where $\pmb{v}$ denotes the velocity, $p$ is the pressure, and $\pmb{f}$ is the prescribed forcing.
We impose periodic boundary conditions on $\Omega=[0,1)^2$:
\begin{equation}
(\pmb{v},p)(x_1+1,x_2,t)=(\pmb{v},p)(x_1,x_2,t),\qquad
(\pmb{v},p)(x_1,x_2+1,t)=(\pmb{v},p)(x_1,x_2,t).
\label{eq:ns_periodic_bc}
\end{equation}

The Reynolds number $\mathrm{Re}$ is the dimensionless parameter whose reciprocal scales the viscous term.
The evaluation spans $\mathrm{Re}\in\{100,500,1000,5000\}$.
The domain, forcing, initial state, and boundary conditions remain fixed across the sweep.

\paragraph{\textbf{Parameterized Helmholtz equations.}}
We consider the two-dimensional Helmholtz family on $\Omega=[-1,1]^2$,
\begin{equation}
\Delta u(x,y;\pmb{\mu})+k^2u(x,y;\pmb{\mu})
-q(x,y;\pmb{\mu})=0,
\qquad \pmb{\mu}=(a_1,a_2),
\label{eq:helmholtz_benchmark}
\end{equation}
where $k$ is the wave number and $(a_1,a_2)$ determine the spatial frequencies of the source.
Following the manufactured benchmark in~\cite{mcclenny2023self}, we prescribe
\begin{equation}
\begin{aligned}
q(x,y;\pmb{\mu})
&=\left[k^2-(a_1\pi)^2-(a_2\pi)^2\right]
  \sin(a_1\pi x)\sin(a_2\pi y),\\
u^*(x,y;\pmb{\mu})
&=\sin(a_1\pi x)\sin(a_2\pi y).
\end{aligned}
\label{eq:helmholtz_manufactured}
\end{equation}

The experiments fix $k=1$ and consider positive integer frequency pairs.
The in-parameter region spans $1\le a_1,a_2\le5$.
Parameter extrapolation covers pairs satisfying $\max(a_1,a_2)>5$.
Both regions include diagonal and off-diagonal frequency combinations.
The exact field $u^*$ supplies the Dirichlet data on $\partial\Omega$ and the evaluation reference.
Each parameter configuration contains 1,000 interior collocation points, 400 boundary points, and 100 test points.

\paragraph{\textbf{Dataset splits and time extrapolation.}}
For the time-dependent CDR and Navier--Stokes benchmarks, the temporal domain is divided into a training horizon $[0,T_{\mathrm{tr}}]$ and an extended inference horizon $[0,T_{\mathrm{te}}]$, where $T_{\mathrm{te}}>T_{\mathrm{tr}}$~\cite{yin2023dino,wang2025pido}.
As defined in Eq.~\ref{eq:param_split_main}, $\mathcal{P}_{\mathrm{interp}}$ contains unseen parameters inside $\operatorname{conv}(\mathcal{P}_{\mathrm{train}})$, whereas $\mathcal{P}_{\mathrm{extra}}$ contains parameters outside $\operatorname{conv}(\mathcal{P}_{\mathrm{train}})$.
Combining the parameter and temporal partitions yields the four regimes in Table~\ref{tab:evaluation_regimes_dldmf}.
The four regimes apply only to the evolution benchmarks.
Helmholtz queries are classified solely along the parameter axis.
S4 imposes simultaneous parameter and temporal shifts and is analyzed separately from aggregate test scores.

\begin{table}[t]
\centering
\caption{Parameter--time evaluation regimes for the time-dependent benchmarks.}
\label{tab:evaluation_regimes_dldmf}
\begin{tabularx}{0.8\linewidth}{@{}c>{\centering\arraybackslash}X>{\centering\arraybackslash}Xl@{}}
\toprule
Regime & Parameter domain & Temporal domain & Purpose \\
\midrule
S1 & $\mathcal{P}_{\mathrm{interp}}$ & $[0,T_{\mathrm{tr}}]$ & Parameter--time interpolation \\
S2 & $\mathcal{P}_{\mathrm{interp}}$ & $(T_{\mathrm{tr}},T_{\mathrm{te}}]$ & Temporal extrapolation \\
S3 & $\mathcal{P}_{\mathrm{extra}}$ & $[0,T_{\mathrm{tr}}]$ & Parameter extrapolation \\
S4 & $\mathcal{P}_{\mathrm{extra}}$ & $(T_{\mathrm{tr}},T_{\mathrm{te}}]$ & Joint parameter--time extrapolation \\
\bottomrule
\end{tabularx}
\end{table}

\subsection{Baselines}
\label{sec:baselines}

The empirical evaluation groups baselines by modeling principle.
The collocation-based PINN baseline minimizes PDE residual and initial- and boundary-condition losses over sampled training points.
Parameter-conditioned physics-informed models are represented by P$^2$INN~\cite{cho2024parameterized}, which encodes PDE parameters into a solution manifold.

Latent dynamics models include DINO~\cite{yin2023dino} and PIDO~\cite{wang2025pido}, both of which evolve latent representations through Neural ODEs~\cite{chen2018neural}.
These methods emphasize continuous-time forecasting, although instance-specific inference may involve auto-decoding mechanisms~\cite{park2019deepsdf}.
Meta-learning solvers are represented by MAD~\cite{huang2022meta}, which adapts across task distributions through latent optimization rather than a globally shared parameter-conditioned dynamics manifold.

The baselines rerun in this study use the stated parameter and time splits and evaluation grid.
Published values reproduced from prior work are identified in the corresponding table notes.
Method-specific differences in training data, optimization, ODE solvers, and test-time adaptation are not treated as equal computation budgets.

\subsection{Evaluation and metrics}
\label{sec:metrics}

\paragraph{\textbf{Accuracy metrics.}}
We quantify solution quality by the $L_2$ relative error \( \|\hat{\mathbf{u}}-\mathbf{u}\|_2/\|\mathbf{u}\|_2 \) and, when informative, the absolute $L_2$ error.
Errors are computed on a dense evaluation grid over $(\pmb{x},t)$ for each test parameter configuration and then aggregated across the test set.
For regime $q\in\{\mathrm{S1},\mathrm{S2},\mathrm{S3},\mathrm{S4}\}$, let $\mathcal{P}_q$ and $\mathcal{T}_q$ denote the finite sets of evaluation parameters and times.
Unless stated otherwise, the split-level metric is
\begin{equation}
\mathcal{E}_q(\Theta)=
\frac{1}{|\mathcal{P}_q|}
\sum_{\pmb{\mu}\in\mathcal{P}_q}
\frac{1}{|\mathcal{T}_q|}
\sum_{t\in\mathcal{T}_q}
\mathrm{Rel}\text{-}L^2(t,\pmb{\mu}).
\label{eq:regime_aggregation_main}
\end{equation}

The parameter-wise outer average assigns equal weight to each parameter configuration.
When independently trained runs are available, Eq.~\ref{eq:regime_aggregation_main} is first evaluated within each run and then summarized across training seeds.
To separate training-window fitting from finite-horizon extrapolation, we evaluate errors over \textit{In-t} ($t\le T_{\mathrm{tr}}$) and \textit{Out-t} ($t>T_{\mathrm{tr}}$) intervals.

In continuous notation, the main field metric is
\begin{equation}
\mathrm{Rel}\text{-}L^2(t,\pmb{\mu})=
\frac{\|\hat u_{\Theta}(\cdot,t;\pmb{\mu})-u^*(\cdot,t;\pmb{\mu})\|_{L^2(\Omega)}}
{\|u^*(\cdot,t;\pmb{\mu})\|_{L^2(\Omega)}+\epsden}.
\label{eq:rel_l2_metric_main}
\end{equation}

The positive scale $\epsden$ matches the denominator scale in Eq.~\ref{eq:supervised_loss_main}, where $\epsden^2$ appears in the denominator of the squared training loss.

\paragraph{\textbf{Temporal extrapolation (\textit{Out-t}).}}
DLDMF produces $\hat u_{\Theta}(\pmb{x},t;\pmb{\mu})$ at requested times by numerically integrating the parameter-conditioned latent ODE and decoding through the fusion interface, so \textit{Out-t} evaluation covers $t\in(T_{\mathrm{tr}},T_{\mathrm{te}}]$.
For baselines designed for a fixed output horizon, temporal extrapolation follows iterative window extension.
The final predicted state in each window initializes the next window until the inference horizon $T_{\mathrm{te}}$ is reached.

\paragraph{\textbf{Parameter generalization.}}
We assess generalization to unseen parameters via two regimes: interpolation inside the training range and extrapolation outside the training range, following prior parameterized PDE evaluations.
The evaluation is conducted for both \textit{In-t} and \textit{Out-t} horizons to quantify how temporal errors change under parameter shifts.


\subsection{Implementation details}
\label{sec:impl}

\paragraph{\textbf{Training supervision.}}
DLDMF uses PDE residuals and initial and boundary constraints; the training objective also permits interior solution supervision through $\lambda_u\mathcal{L}_u$ in Eq.~\ref{eq:loss_dldmf}.
For parameterized training, we mini-batch over PDE parameter instances and the collocation points associated with each instance, so that each optimization step mixes multiple $\pmb{\mu}$ values and encourages a shared parameterized representation.

\paragraph{\textbf{Neural ODE solver.}}
A differentiable ODE solver integrates the latent dynamics in continuous time, permitting evaluation at requested temporal resolutions and beyond the training horizon.
Solver type and tolerances are compared only for ODE-based models rerun under the same evaluation protocol.

\paragraph{\textbf{Reproducibility.}}
Unless a caption explicitly states a multi-seed aggregation, the displayed values are point estimates from the recorded experimental run.
Network capacity, training budgets, and inference settings are documented for each method so that remaining structural differences are visible.
During training, we record the objective components, gradient norm, Out-$t$ validation error, and average latent-Jacobian norm $\|\partial f_{\theta_f}/\partial \pmb{z}\|_F$.
These quantities serve as training-monitoring diagnostics, while Proposition~\ref{prop:main_temporal_perturbation} uses uniform theoretical constants.


\subsection{Main results}
\label{sec:main_results}

MAD is extended beyond the native output horizon via the iterative protocol in Sec.~\ref{sec:metrics}, following~\cite{wang2023long}.
Table~\ref{tab:main_results} compares fitting instances with held-out test instances.
The parameter-wise studies below further examine parameter and temporal shifts.
Among the values listed in Table~\ref{tab:main_results}, DLDMF has the lowest Train Out-$t$, Test In-$t$, and Test Out-$t$ errors, including 1.89\% and 4.21\% on the held-out test set.
P$^2$INN yields 21.34\% and 32.87\%, consistent with Fig.~\ref{fig:p2inn_extrapolation}.
PIDO also yields larger In-$t$ and Out-$t$ $L_2$ errors than DLDMF on both the training and test sets.
Reducing the DINO training fraction lowers Train In-$t$ error but increases Out-$t$ and test errors, indicating weaker generalization.
The different fitting and extrapolation trends distinguish accuracy within the training window from generalization across instances and time.
DLDMF maintains low errors across these evaluation splits, supporting its use as a shared surrogate for the evaluated solution family.

\begin{table}[t]
\caption{Relative $L_2$ errors (\%) for DLDMF and baselines on the CDR benchmark under the temporal evaluation.}
\label{tab:main_results}
\begin{center}
\begin{small}
\begin{sc}
\resizebox{0.65\linewidth}{!}{
\begin{tabular}{ccccccc}
\toprule
\multirow{2}{*}{Model} & \multirow{2}{*}{Training fraction} & \multicolumn{2}{c}{Train} & & \multicolumn{2}{c}{Test} \\ \cmidrule{3-4} \cmidrule{6-7}
 & & In-t & Out-t & & In-t & Out-t \\
\midrule
\name & - & 1.35 & \textbf{3.15} & & \textbf{1.89} & \textbf{4.21} \\
P$^2$INN & - & 8.45 & 15.23 & & 21.34 & 32.87 \\
PIDO & - & 1.58 & 3.89 & & 5.67 & 8.94\\
DINO & 100\% & 1.42 & 3.32 & & 4.89 & 6.52\\
\midrule
DINO & 50\% & 1.14 & 4.90 & & 5.25 & 8.76 \\
DINO & 25\% & 0.82 & 7.52 & & 7.37 & 13.58\\
DINO & 12.5\% & \textbf{0.47} & 13.59 & & 15.12 & 31.74 \\
\bottomrule
\end{tabular}}
\end{sc}
\end{small}
\par\vspace{2pt}
\begin{minipage}{0.65\linewidth}
\footnotesize\textit{Note:} The DINO and PIDO entries are external reference values, not equal-budget reruns under the DLDMF training protocol.
The PIDO is described in~\cite{wang2025pido}.
\end{minipage}
\end{center}
\end{table}

\begin{figure}[H]
\centering
\includegraphics[width=0.85\linewidth]{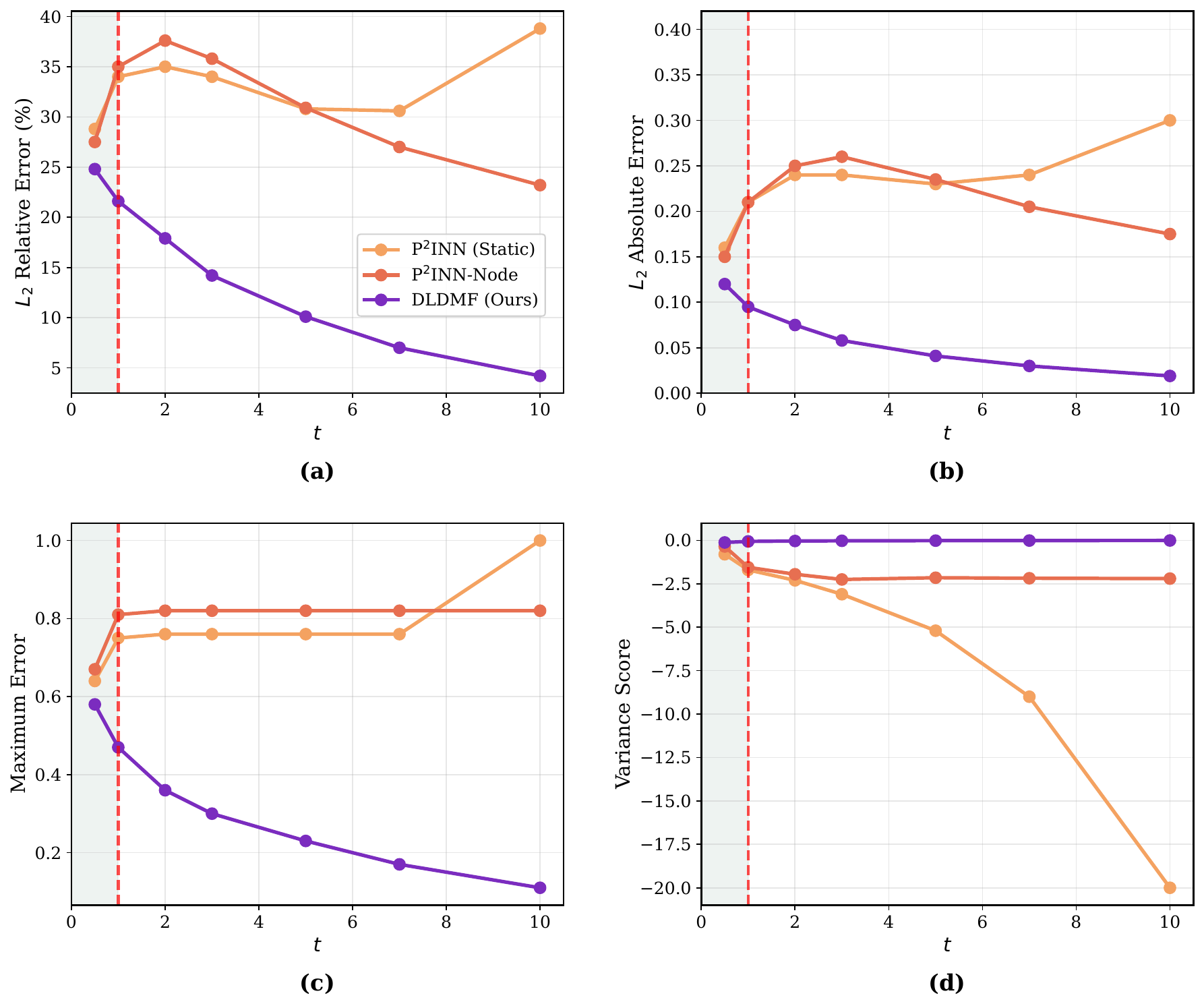}
\caption{Temporal extrapolation metrics for DLDMF and P$^2$INN.
(a) Reference-normalized $L_2$ error. (b) Absolute $L_2$ error. (c) Maximum pointwise prediction error. (d) Variance score across the training and extrapolation intervals.}
\label{fig:dldmf_p2inn_error}
\end{figure}

Fig.~\ref{fig:dldmf_p2inn_error} compares DLDMF with the P$^2$INN variants over the temporal extrapolation horizon across four error metrics.
Beyond the training boundary at $t=1$, DLDMF maintains lower relative, absolute, and maximum errors in the displayed comparison.
The result highlights the performance contrast between static coordinate-time regression and parameter-conditioned latent evolution under the inference settings of the two methods.
The agreement across the relative, absolute, and maximum-error measures indicates improvements in both overall reconstruction and peak prediction discrepancies.

\subsection{Ablation studies}
\label{sec:ablation_exp}

We align ablations with the fusion-interface decomposition in Proposition~\ref{prop:main_fusion_decomp}.
Each variant modifies one model component or numerical treatment while the remaining training protocol is fixed.

\paragraph{\textbf{Temporal-evolution mechanism.}}
Fig.~\ref{fig:dldmf_evolution_mechanism_modules} reports the relative error trajectories of the four temporal variants for $\beta=3$ and $\beta=15$; the out-of-time region begins at $t=1$.
Fig.~\ref{fig:dldmf_evolution_mechanism_efficiency} presents parameter counts, elapsed inference time, and relative errors of the four temporal variants under the stated protocol.
The error trajectories remain within the same order of magnitude after $t=1$, although the shifted $\beta=15$ case shows greater fluctuations.
Neural ODE gives the lowest mean relative $L_2$ error (1.01\%) with 0.24 M parameters and 1941 s elapsed inference time, whereas Latent ODE yields 1.15\% with 0.20 M parameters and 51 s.
Hypersolver uses the fewest parameters (0.16 M) but has the highest mean error (1.44\%) and 2778 s elapsed inference time.
And Neural Flow uses 0.40 M parameters and 4319 s for a 1.24\% error.
The smallest model, the fastest evaluation, and the lowest prediction error correspond to different temporal modules.
These results favor Neural ODE when prediction accuracy is prioritized and Latent ODE when shorter evaluation time is required under the tested protocol.

\begin{figure}[H]
\centering
\includegraphics[width=0.9\linewidth]{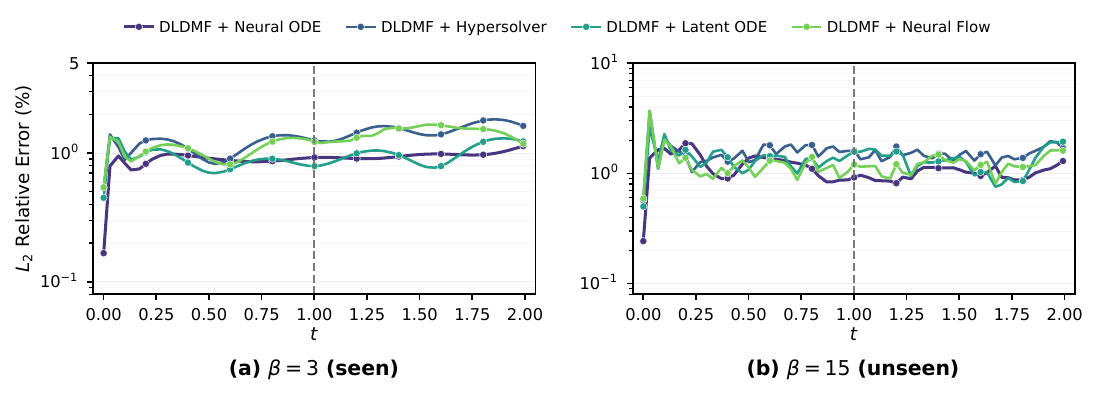}
\caption{DLDMF temporal-module errors.
(a) $\beta=3$ (seen). (b) $\beta=15$ (unseen).}
\label{fig:dldmf_evolution_mechanism_modules}
\end{figure}

\begin{figure}[H]
\centering
\includegraphics[width=0.9\linewidth]{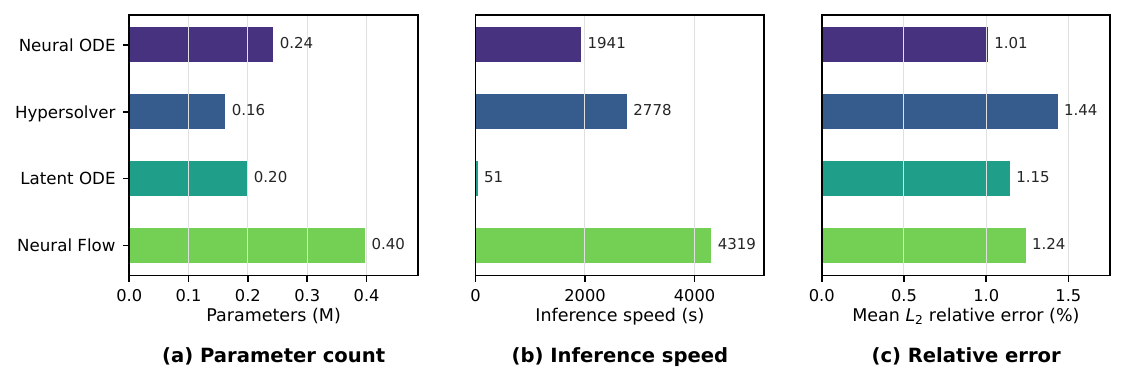}
\caption{Computational comparison of temporal-evolution modules in DLDMF.
Panels (a)--(c) show the parameter count, elapsed inference time, and mean relative $L^2$ error, respectively.
The quantity in panel (b) is labelled in seconds.}
\label{fig:dldmf_evolution_mechanism_efficiency}
\end{figure}

\paragraph{\textbf{Fast adaptation with decoder modulation.}}
Fig.~\ref{fig:decoder_modulation} compares alternative decoder-side fine-tuning decompositions for adapting a pretrained DLDMF to a target parameter.
The comparison contrasts structured decoder modulation in activation space (Shift, Scale, and FiLM), spectral weight space (GFM), and SVD singular directions with full-decoder updating and fixed-pretrained evaluation.
After 100 fine-tuning steps, the SVD-based update gives the lowest final error for both the seen parameter $\beta=3$ and the shifted parameter $\beta=15$.
The comparison with full decoder fine-tuning suggests that the pretrained singular directions provide an effective low-dimensional space for adaptation within the evaluated optimization budget.
The SVD update adjusts gains along pretrained singular directions (Eq.~\ref{eq:eq_svd_adapt}), and the lower final errors in both settings are consistent with useful reconstruction directions already present in the pretrained decoder.
The comparison supports singular-direction modulation as an optional adaptation mechanism at both tested parameters, including the shifted $\beta=15$ setting.
The core DLDMF solver remains the feed-forward parameter-conditioned latent time-integration model.

\begin{figure}[H]
\centering
\includegraphics[width=0.9\linewidth]{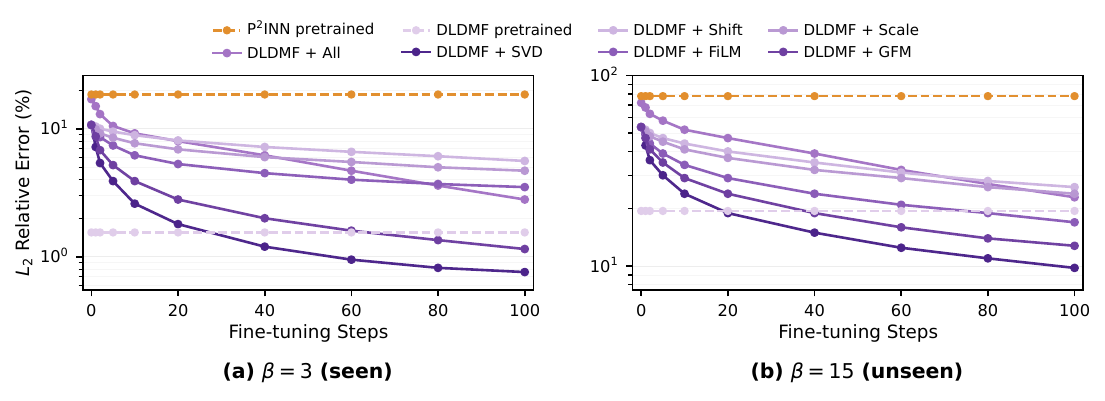}
\caption{DLDMF modulation fine-tuning.
(a) $\beta=3$ (seen). (b) $\beta=15$ (unseen).}
\label{fig:decoder_modulation}
\end{figure}

\subsection{Computational efficiency}
\label{sec:eff}

To characterize the computational implications of direct multi-query evaluation, we measure the efficiency metrics alongside the accuracy.
Specifically, we measure (i) inference latency per parameter query, (ii) the number of test-time gradient updates for methods that require auto-decoding, and (iii) memory overhead associated with storing per-instance latent codes in instance-parameterized methods.
For direct-inference methods, no test-time gradient update is performed; the value one in Fig.~\ref{fig:computational_efficiency}(b) is used only as a plotting floor for visualization.
The efficiency analysis complements the accuracy metrics by quantifying deployment cost in repeated-query settings.

\begin{table}[t]
\centering
\small
\setlength{\tabcolsep}{4.0pt}
\caption{Accuracy-efficiency tradeoff for out-of-time convection queries.}
\label{tab:accuracy_efficiency_out_time}
\begin{tabular}{ccccc}
\toprule
Model & Out-$t$ rel. $L_2$ (\%) & Latency/query (ms) & Peak memory (MB) & Parameters \\
\midrule
P$^2$INN & 41.75 & \textbf{0.34} & \textbf{29.49} & 76851 \\
DLDMF & \textbf{0.14} & 1.96 & 30.05 & \textbf{76059} \\
\bottomrule
\end{tabular}
\end{table}

Table~\ref{tab:accuracy_efficiency_out_time} shows lower Out-$t$ error for DLDMF, whereas P$^2$INN has lower latency and peak memory in the out-of-time convection query setting.
With comparable parameter counts and peak memory usage, the main trade-off in this query setting is between prediction accuracy and evaluation time.
Fig.~\ref{fig:computational_efficiency} presents an additional
five-method comparison across the evaluated benchmark suite, highlighting relative inference latency, optimization steps, and memory footprints.


\begin{figure}[H]
\centering
\includegraphics[width=\linewidth]{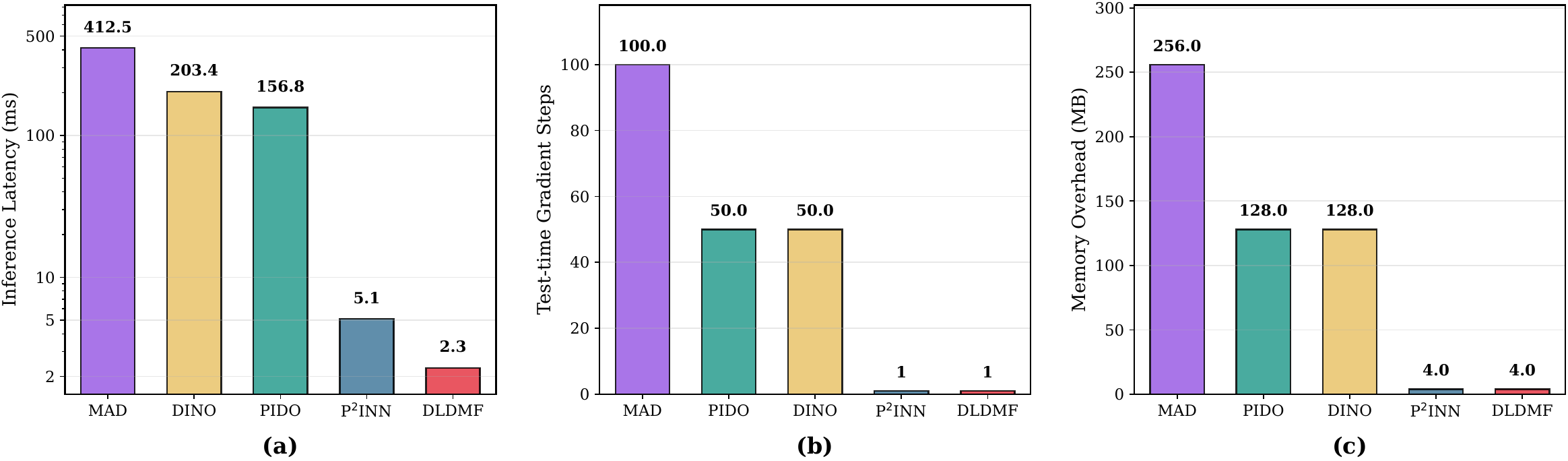}
\caption{Computational efficiency across baselines.
(a) Inference latency. (b) Test-time gradient steps. (c) Memory overhead.}
\label{fig:computational_efficiency}
\end{figure}

\subsection{Qualitative temporal extrapolation and collocation-density sensitivity}
\label{sec:qualitative_results}

The qualitative visualizations in the section complement the aggregated errors in Table~\ref{tab:main_results} by showing how DLDMF behaves along full temporal trajectories.
Collectively, the visualizations provide evidence on temporal-phase preservation, finite-horizon error growth, and sensitivity to the residual and anchor collocation budgets.

\paragraph{\textbf{Spatiotemporal structure across parameter and temporal shifts.}}
Fig.~\ref{fig:cdr_time_evolution_beta} compares the spatiotemporal solution distributions for increasing convection parameters and extends the inference interval to five times the training horizon.
The dashed line marks $T_{\mathrm{tr}}=1$.
The cases $\beta=1$ and $\beta=10$ correspond to in-parameter settings, whereas $\beta=30$ represents a large out-of-parameter shift.
DLDMF retains the phase, orientation, and amplitude of the periodic transport bands across the displayed horizon.
P$^2$INN loses the oscillatory structure beyond the training window, while PIDO exhibits increasing attenuation and phase distortion as $\beta$ increases.
Under the joint shift $\beta=30$ and $t>T_{\mathrm{tr}}$, the DLDMF solution distribution remains visually aligned with the reference solution.
The comparison provides qualitative evidence of parameter-conditioned temporal extrapolation and complements the relative-error results in Table~\ref{tab:main_results}.

\begin{figure}[H]
\centering
\includegraphics[width=0.98\linewidth]{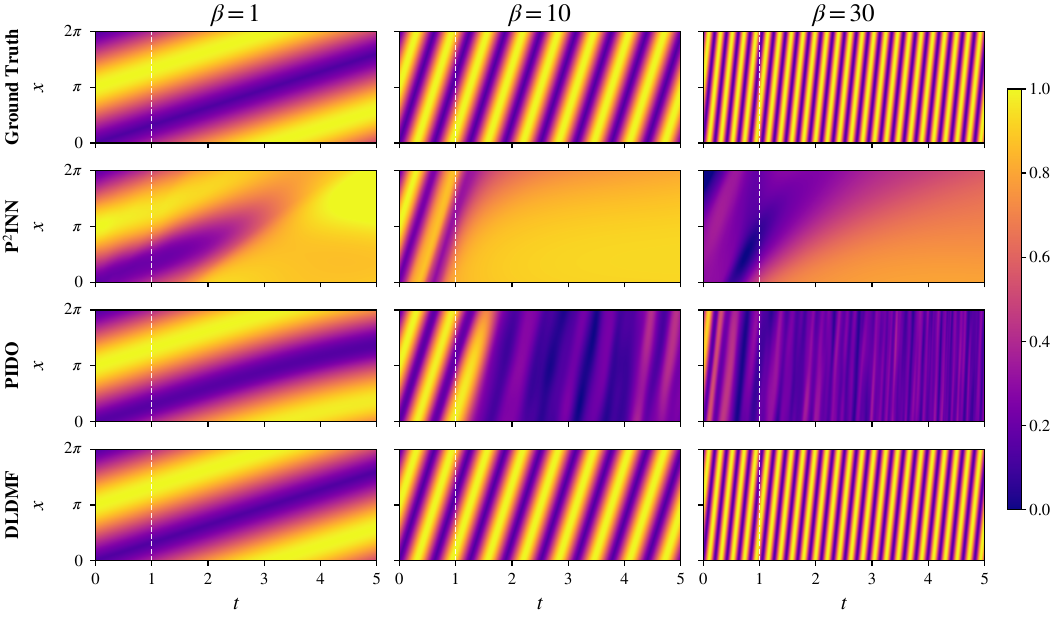}
\caption{Spatiotemporal solution distributions for convection across parameter and temporal shifts.}
\label{fig:cdr_time_evolution_beta}
\end{figure}

Fig.~\ref{fig:completed_quadrants_distribution} separates parameter interpolation from extrapolation and training-time prediction from temporal extrapolation.
For $\beta=1$, DLDMF retains the translated solution structure beyond $T_{\mathrm{tr}}=1.0$.
P$^2$INN shows increasing attenuation at later times.
For $\beta=11$, DLDMF remains visually aligned with the reference solutions.
P$^2$INN and PIDO deviate markedly after the training horizon.
The lower panel represents the joint parameter--time shift.

\begin{table}[t]
\centering
\scriptsize
\setlength{\tabcolsep}{3.5pt}
\caption{Phase-aligned periodic-convection diagnostic under characteristic coordinates (\%).}
\label{tab:param_extrap_baseline}
\resizebox{0.5\linewidth}{!}{%
\begin{tabular}{ccccccc}
\toprule
Model & $\beta=1$ & $\beta=5$ & $\beta=10$ & $\beta=11$ & $\beta=13$ & $\beta=15$ \\
\midrule
\multicolumn{7}{c}{\textit{In-$t$}} \\
MAD & 3.43 & 4.06 & 19.96 & 24.49 & 32.97 & 40.92 \\
P$^2$INN & 4.78 & 3.40 & 11.64 & 15.59 & 32.13 & 46.51 \\
PIDO & 0.79 & 1.50 & 1.64 & 2.05 & 6.97 & 19.00 \\
DLDMF & $\mathbf{0.16}$ & $\mathbf{0.14}$ & $\mathbf{0.15}$ & $\mathbf{0.16}$ & $\mathbf{0.19}$ & $\mathbf{0.25}$ \\
\midrule
\multicolumn{7}{c}{\textit{Out-$t$}} \\
MAD & 26.25 & 22.86 & 43.76 & 44.52 & 44.88 & 44.25 \\
P$^2$INN & 9.88 & 19.87 & 50.36 & 55.72 & 58.33 & 56.33 \\
PIDO & 1.21 & 3.15 & 54.79 & 61.44 & 73.98 & 80.54 \\
DLDMF & $\mathbf{0.16}$ & $\mathbf{0.14}$ & $\mathbf{0.15}$ & $\mathbf{0.16}$ & $\mathbf{0.19}$ & $\mathbf{0.25}$ \\
\bottomrule
\end{tabular}}
\end{table}

Table~\ref{tab:param_extrap_baseline} quantifies the calibrated periodic-convection diagnostic with different $\beta$.
The DLDMF errors remain between $0.14\%$ and $0.25\%$ for both In-$t$ and Out-$t$ queries.
At $\beta=15$, the Out-$t$ errors of MAD, P$^2$INN, and PIDO reach $44.25\%$, $56.33\%$, and $80.54\%$, respectively.
For diagnostic purposes, the DLDMF entries in
Table~\ref{tab:param_extrap_baseline} are evaluated in characteristic coordinates $x-\beta t$.
This diagnostic assesses phase-aligned
reconstruction accuracy, whereas the baseline entries reflect unaligned end-to-end extrapolation.
The matching In-$t$ and Out-$t$ errors at the displayed precision indicate comparable reconstruction accuracy across the two intervals after characteristic alignment.
The primary S3--S4 assessment uses direct inference in the original query coordinates.

\begin{figure}[H]
\centering
\includegraphics[width=0.9\linewidth]{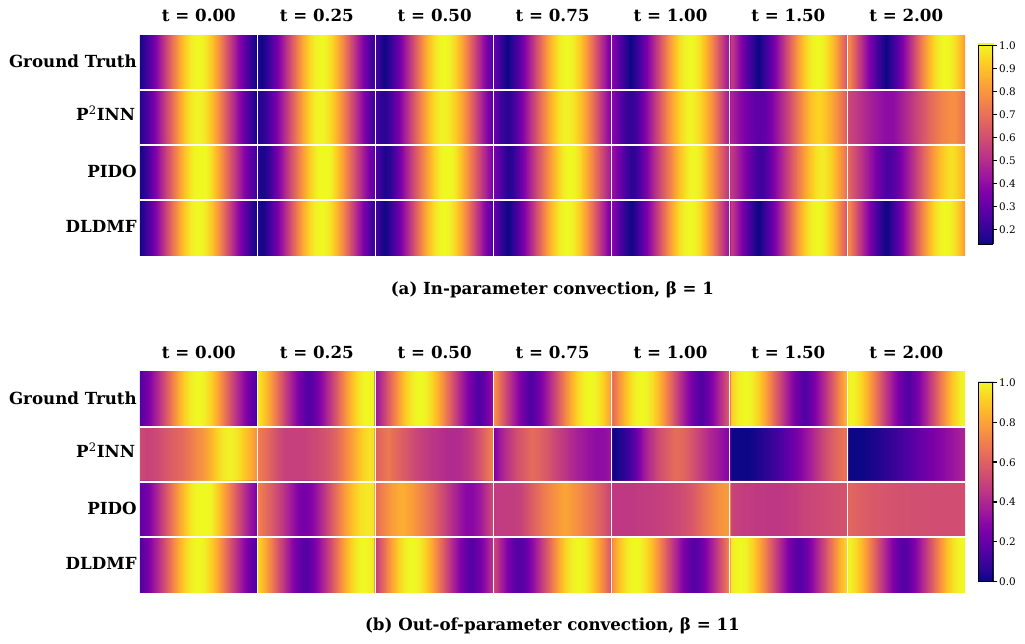}
\caption{Convection solutions across parameter and temporal splits.
(a) In-parameter convection, $\beta=1$. (b) Out-of-parameter convection, $\beta=11$.}
\label{fig:completed_quadrants_distribution}
\end{figure}

Fig.~\ref{fig:cdr_diagonal_2d_parameter} extends the four-quadrant assessment to a two-dimensional coupled CDR benchmark.
The upper and lower panels correspond to in-parameter and out-of-parameter settings, respectively, while the red time labels mark queries beyond the training horizon $T_{\mathrm{tr}}=1.0$.
DLDMF preserves the location, orientation, and intensity of the diagonal solution band across the displayed horizon in both panels.
P$^2$INN develops increasing spatial distortion and attenuation, whereas PIDO loses most of the coherent field structure beyond the training interval.
The separation is most pronounced under the joint parameter and temporal shift, where DLDMF remains visually aligned with the reference field while the baselines exhibit severe attenuation or collapse.
The comparison extends the qualitative evidence from one-dimensional transport to a two-dimensional coupled benchmark, while remaining a visual rather than a quantitative error assessment.

\begin{figure}[H]
\centering
\includegraphics[width=0.85\linewidth]{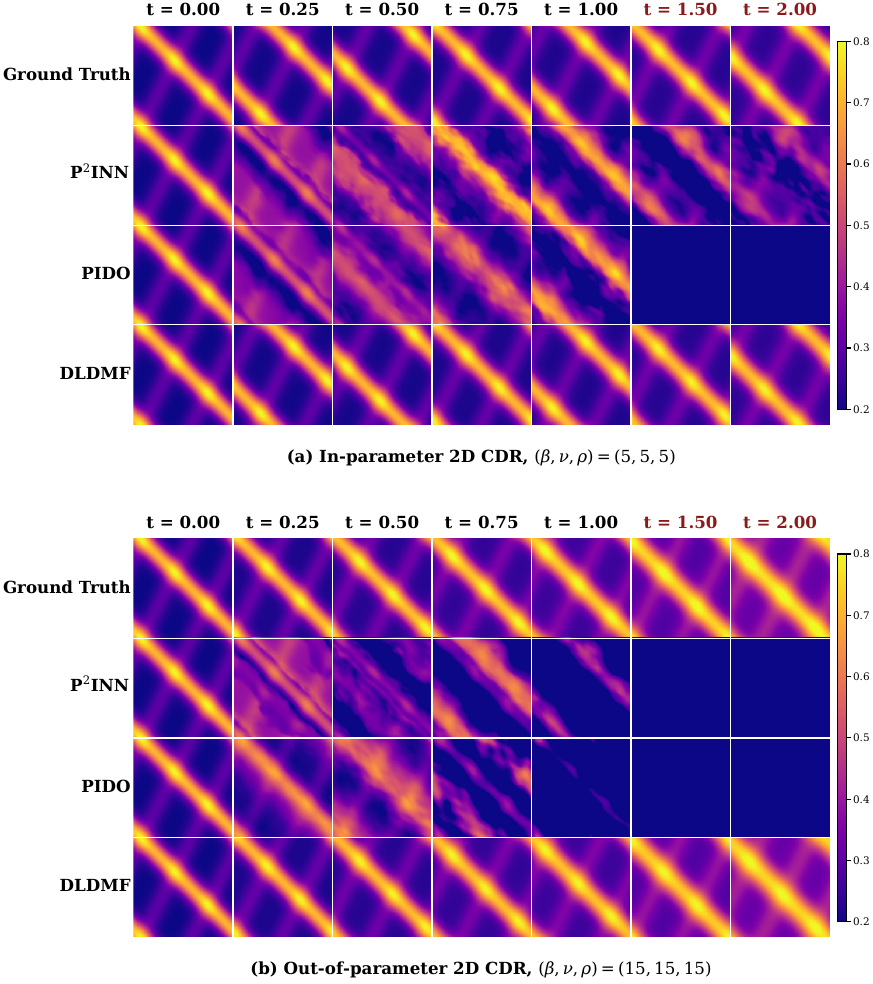}
\caption{Two-dimensional CDR predictions across parameter and temporal splits.
(a) In-parameter, $(\beta,\nu,\rho)=(5,5,5)$. (b) Out-of-parameter, $(15,15,15)$. Red labels mark $t>T_{\mathrm{tr}}=1.0$.}
\label{fig:cdr_diagonal_2d_parameter}
\end{figure}

Fig.~\ref{fig:cdr_diagonal_3d_parameter} further examines the preservation of three-dimensional CDR solution structures under parameter and temporal shifts.
The upper panel presents the in-parameter settings $(\beta,\nu,\rho)=(1,3,5)$, whereas the lower panel applies the out-of-parameter combination $(11,13,15)$.
The snapshots at $t=1.5$ and $t=2.0$ lie beyond the training horizon $T_{\mathrm{tr}}=1.0$.
DLDMF preserves the location and orientation of the diagonal structures across all displayed times in both settings.
The baseline predictions exhibit visible deformation after the training horizon in the in-parameter case.
Under the joint parameter--time shift, no visible P$^2$INN structure appears in the rendering, while PIDO approaches a nearly uniform field.
The visual comparison extends the observed parameter-conditioned temporal generalization to the three-dimensional benchmark, while providing qualitative evidence only.

\begin{figure}[H]
\centering
\includegraphics[width=0.88\linewidth]{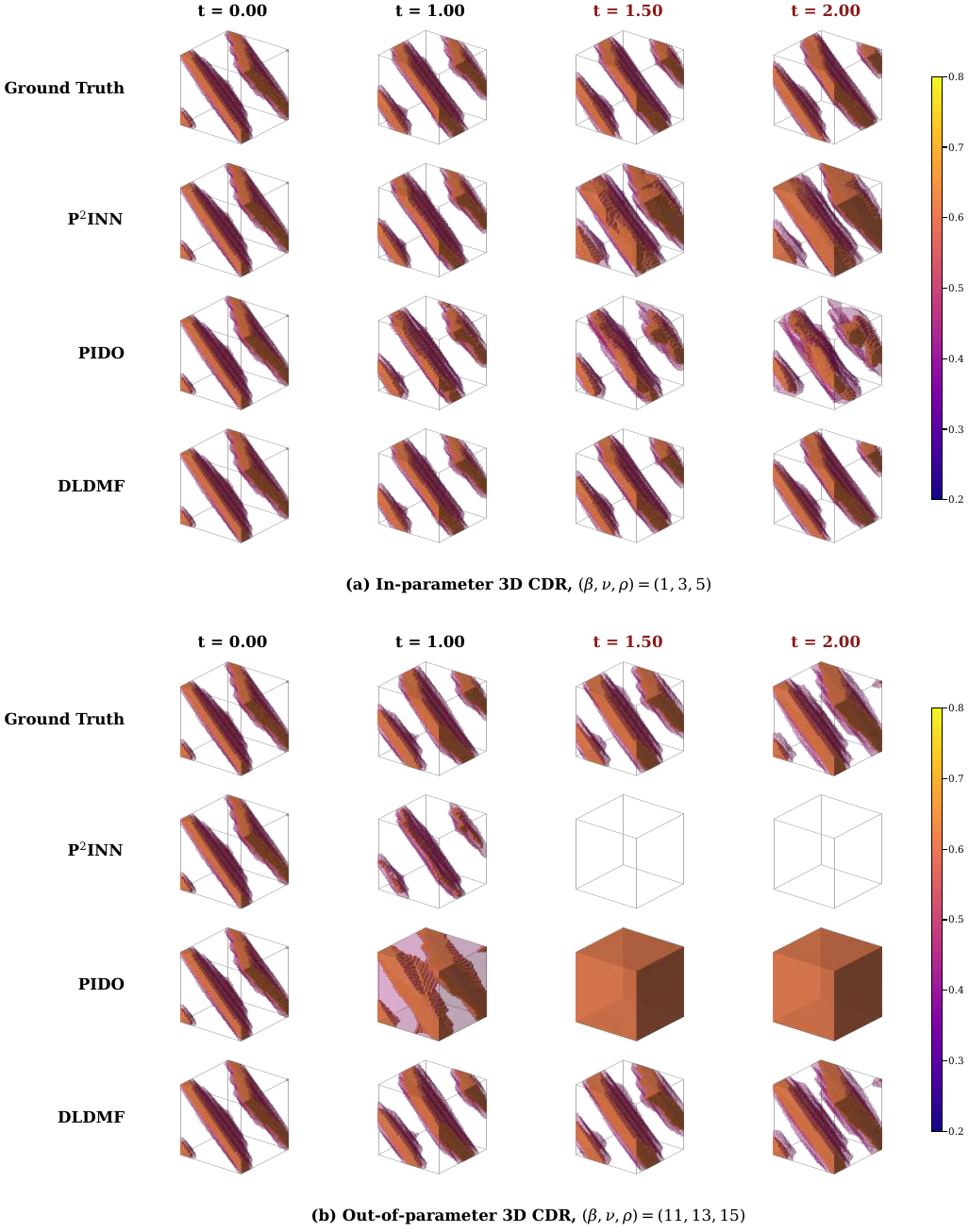}
\caption{Three-dimensional CDR solutions across parameter and temporal splits.
(a) In-parameter, $(\beta,\nu,\rho)=(1,3,5)$. (b) Out-of-parameter, $(11,13,15)$. Times $t=1.5,2.0$ exceed $T_{\mathrm{tr}}=1.0$.}
\label{fig:cdr_diagonal_3d_parameter}
\end{figure}

\paragraph{\textbf{Physics-informed diagnostic with periodic phase alignment.}}
Fig.~\ref{fig:dldmf_profile_predictions} visualizes one-dimensional solution trajectories for the convection equation under increasing convection speed.
The first and second columns present in-parameter settings, $\beta=1$ and $\beta=10$, whereas the third presents the out-of-parameter setting $\beta=15$.
The horizontal axes cover progressively longer inference horizons, up to $T=1.0$, $T=3.0$, and $T=5.0$, respectively.
The displayed horizons cover temporal extrapolation and parameter extrapolation in addition to interpolation within the training window.

The exact and predicted solution trajectories exhibit similar periodic structures across all three columns.
The phase-aligned diagnostic combines a calibrated DLDMF checkpoint with the periodic phase structure of the pure convection equation to align extrapolated-time queries.
As shown in Table~\ref{tab:profile_prediction_metrics}, the relative $L_2$ errors were 3.47\%/5.30\% for in-time/out-time prediction at $\beta=1$, 2.77\%/2.86\% at $\beta=10$, and 2.72\%/2.65\% at the out-of-parameter setting $\beta=15$.
The close In-$t$/Out-$t$ agreement at $\beta=10$ and $\beta=15$ shows that the calibrated predictions retain comparable accuracy across the evaluated time intervals.
With periodic phase alignment imposed, the reconstructed solutions remain close to the aligned reference solutions over the displayed horizons.
Fig.~\ref{fig:dldmf_profile_predictions} is interpreted as a calibrated diagnostic of the learned solution manifold.

\begin{table}[t]
\centering
\small
\caption{Convection diagnostic metrics across parameter settings.}
\label{tab:profile_prediction_metrics}
\begin{tabular}{ccccc}
\toprule
Parameter split & $\beta$ & Plotted horizon & In-$t$ rel. $L_2$ (\%) & Out-$t$ rel. $L_2$ (\%) \\
\midrule
In-$\mu$ & 1 & 1.0 & 3.47 & 5.30 \\
In-$\mu$ & 10 & 3.0 & 2.77 & 2.86 \\
Out-$\mu$ & 15 & 5.0 & 2.72 & 2.65 \\
\bottomrule
\end{tabular}
\par\vspace{2pt}
\footnotesize\textit{Note:} The plotted horizon denotes the time extent of the corresponding solution visualization.
\end{table}

\begin{figure}[H]
\centering
\includegraphics[width=0.95\linewidth]{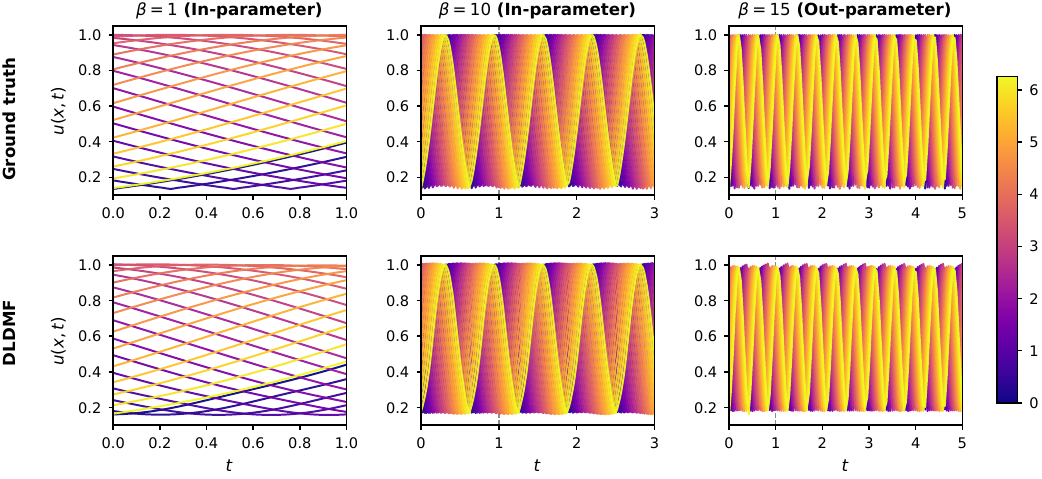}
\caption{Phase-aligned DLDMF diagnostic for parameterized convection dynamics with periodic phase alignment.}
\label{fig:dldmf_profile_predictions}
\end{figure}

\paragraph{\textbf{Collocation-density sensitivity.}}
We next evaluate the sensitivity of DLDMF with respect to the density of physics-informed training points.
Figs.~\ref{fig:dldmf_sample_eff_in_mu} and~\ref{fig:dldmf_sample_eff_out_mu} compare models trained with 5\%, 25\%, and 100\% of the residual and anchor points.
Anchors denote non-residual samples carrying initial, boundary, or supervised solution constraints.
A fixed dense evaluation grid isolates the effect of the available training signal on the visual predictions.

Table~\ref{tab:sample_efficiency_metrics} shows the In-$t$ and Out-$t$ errors for the In-$\mu$ and Out-$\mu$ cases.
At In-$\mu$, the predicted temporal evolution remains coherent at 5\% sampling, and the In-$t$/Out-$t$ errors increase only from 0.103\%/0.118\% at 100\% sampling to 0.140\%/0.152\% at 5\% sampling.
At Out-$\mu$, the In-$t$/Out-$t$ errors are higher, ranging from 0.128\%/0.313\% at 100\% sampling to 0.201\%/0.379\% at 5\% sampling.
Additional collocation points improve both cases, but the parameter shift remains the more difficult setting.
At every sampling rate, the Out-$\mu$ case has a larger In-$t$/Out-$t$ gap than the In-$\mu$ case.
This persistent separation identifies parameter variation as an additional influence on temporal prediction alongside collocation density.



\begin{table}[t]
\centering
\small
\caption{Prediction errors across physics-informed collocation rates for parameterized temporal extrapolation.}
\label{tab:sample_efficiency_metrics}
\resizebox{.7\linewidth}{!}{%
\begin{tabular}{cccc}
\toprule
Parameter split & Collocation rate & In-$t$ rel. $L_2$ (\%) &
Out-$t$ rel. $L_2$ (\%) \\
\midrule
In-$\mu$  & 5\%   & 0.140 & 0.152 \\
In-$\mu$  & 25\%  & 0.112 & 0.126 \\
In-$\mu$  & 100\% & 0.103 & 0.118 \\
\midrule
Out-$\mu$ & 5\%   & 0.201 & 0.379 \\
Out-$\mu$ & 25\%  & 0.147 & 0.321 \\
Out-$\mu$ & 100\% & 0.128 & 0.313 \\
\bottomrule
\end{tabular}
}
\end{table}

\begin{figure}[H]
\centering
\includegraphics[width=0.75\linewidth]{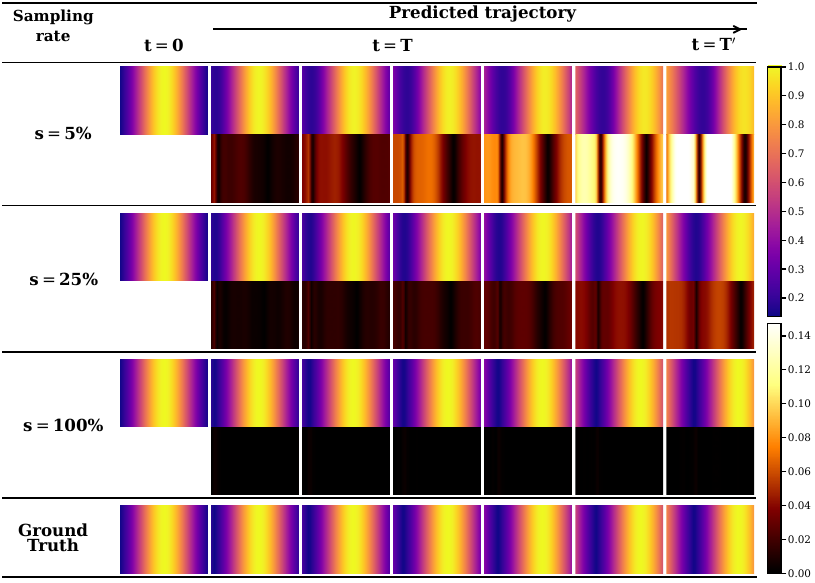}
\caption{In-parameter predictions at 5\%, 25\%, and 100\% collocation-point densities.}
\label{fig:dldmf_sample_eff_in_mu}
\end{figure}

\paragraph{\textbf{Train-trajectory and test-trajectory temporal extrapolation.}}
Figs.~\ref{fig:dldmf_sst_tr} and~\ref{fig:dldmf_sst_ts} visualize DLDMF predictions on representative train and test trajectories.
The train-trajectory visualization presents parameter settings seen during training, whereas the test-trajectory visualization presents shifted parameter settings.
Both figures compare reference states and DLDMF predictions at multiple time points, with the later frames lying outside the training time window.

On train trajectories, DLDMF produced low errors for both convection and diffusion dynamics.
For convection, the mean relative $L_2$ error was 0.15\% in-time and 0.14\% out-time.
For diffusion, the corresponding errors were 0.20\% and 0.22\%.
The final-time errors at $t\approx 2.0$ were 0.15\% for convection and 0.30\% for diffusion.
The values show accurate finite-horizon evolution for the displayed in-parameter configurations.

The test trajectories exhibited a more heterogeneous error pattern.
For convection under the tested parameter shift, DLDMF yielded mean errors of 0.27\% in-time and 0.35\% out-time, with a final-time error of 0.65\%.
For diffusion under the tested parameter setting, the out-time error increased to 12.70\%, with a final-time error of 18.00\%.
The convection and diffusion results show that generalization depends on the shifted dynamics.
Under the tested parameter shift, the diffusion benchmark exhibits larger temporal-extrapolation errors.
The two in-parameter cases retain similar accuracy across the training boundary, while their shifted counterparts show different responses, highlighting operator-dependent sensitivity in joint parameter--time extrapolation.

\begin{figure}[H]
\centering
\includegraphics[width=0.75\linewidth]{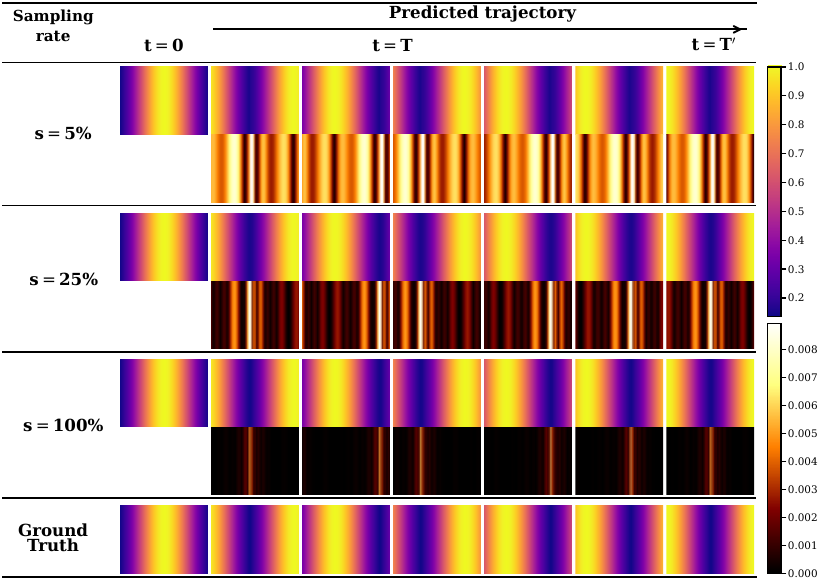}
\caption{Out-of-parameter predictions at 5\%, 25\%, and 100\% collocation-point densities.}
\label{fig:dldmf_sample_eff_out_mu}
\end{figure}

\paragraph{\textbf{Prediction--error strips over time.}}
Figs.~\ref{fig:dldmf_visu_tr_full} and~\ref{fig:dldmf_visu_ts_full} align the predicted fields and the absolute errors over time for the in-parameter and out-of-parameter trajectories.
For the in-parameter trajectory, the predicted solution retained smooth translation and amplitude across the displayed horizon.
The mean out-time relative $L_2$ error was 0.14\%, and the framewise errors remained below 0.16\%.
The localized error bands showed no sustained amplification beyond the training boundary.
Under the parameter shift, DLDMF preserved the dominant moving structure, although the mean out-time error increased to 0.39\% and the final-time error band became more pronounced.
The localized error bands identify the late-time regions most affected by the parameter shift, while DLDMF retains the dominant moving structure.

\begin{figure}[H]
\centering
\includegraphics[width=\linewidth]{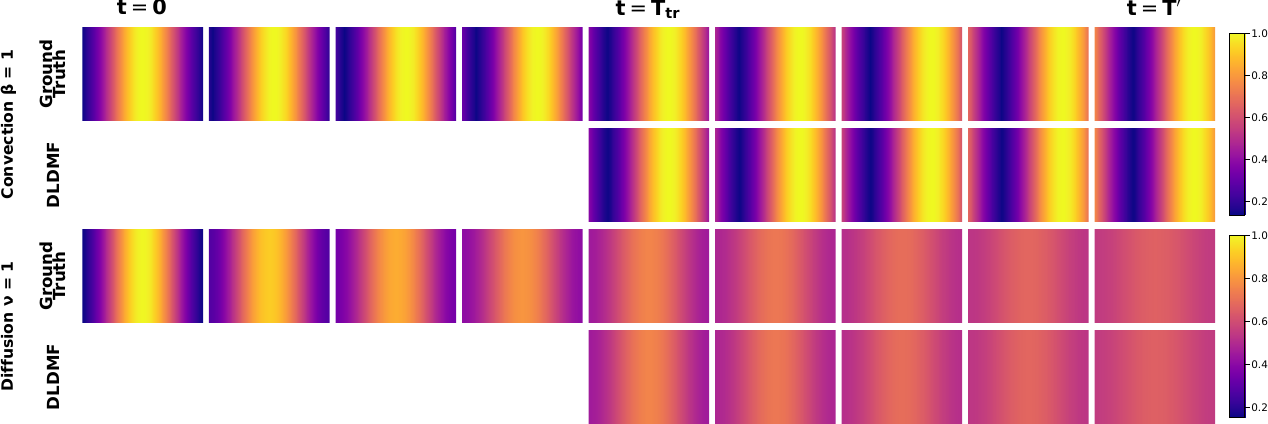}
\caption{DLDMF temporal extrapolation on representative in-parameter training trajectories for CDR dynamics.}
\label{fig:dldmf_sst_tr}
\end{figure}

\begin{figure}[H]
\centering
\includegraphics[width=\linewidth]{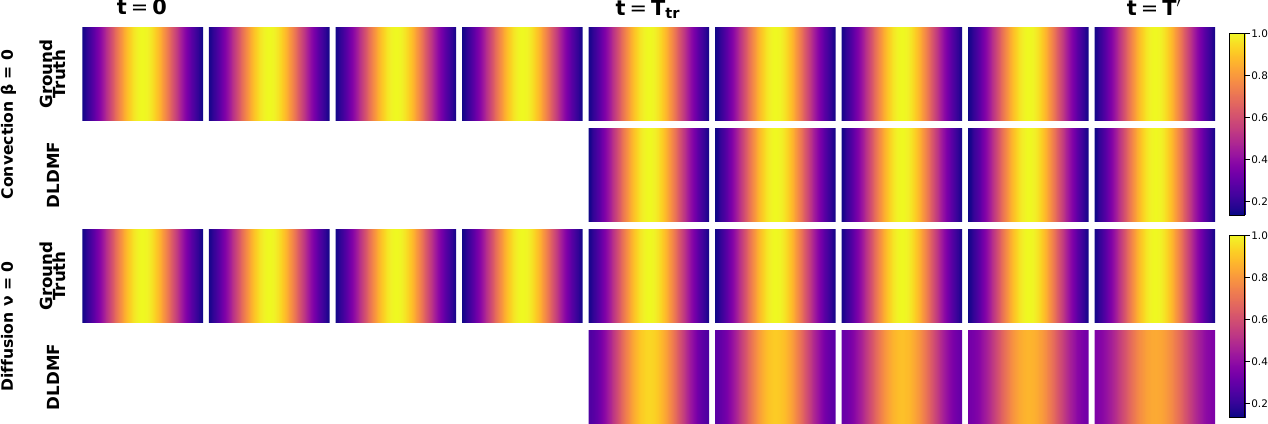}
\caption{DLDMF temporal extrapolation on representative shifted-parameter test trajectories for CDR dynamics.}
\label{fig:dldmf_sst_ts}
\end{figure}

\begin{figure}[H]
\centering
\includegraphics[width=0.85\linewidth]{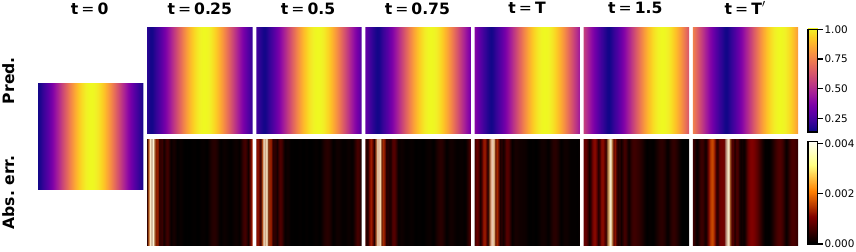}
\caption{Prediction and error strips for an in-parameter temporal trajectory case.}
\label{fig:dldmf_visu_tr_full}
\end{figure}

\begin{figure}[H]
\centering
\includegraphics[width=0.85\linewidth]{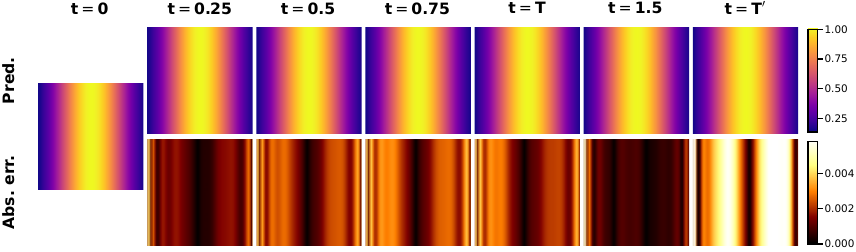}
\caption{Prediction and error strips for a shifted-parameter temporal trajectory case.}
\label{fig:dldmf_visu_ts_full}
\end{figure}

\paragraph{\textbf{Parameter generalization across CDR parameters.}}
Fig.~\ref{fig:param_generalization} presents the relative error curves of DLDMF and P$^2$INN across the tested convection, diffusion, and reaction parameters.
Across all three one-dimensional parameter sweeps, DLDMF maintained a much flatter error curve than P$^2$INN.
For the convection sweep, the DLDMF relative $L_2$ error increased only from approximately 2.55\% to 5.27\%, whereas the errors yielded by P$^2$INN increased from 5.80\% to 17.90\%.
The gap became larger for diffusion and reaction: at the largest tested parameter values, DLDMF reached 5.62\% for $\nu=10$ and 3.59\% for $\rho=10$, while P$^2$INN reached 28.40\% and 44.90\%, respectively.
This consistency supports repeated solution queries across each tested coefficient range using the shared parameter-conditioned model.
The reaction result is particularly informative because the P$^2$INN error grows almost monotonically with $\rho$, whereas DLDMF remains within a narrow low-error band, showing a DLDMF advantage in the reaction-dominated regime.

\begin{figure}[H]
\centering
\includegraphics[width=0.95\linewidth]{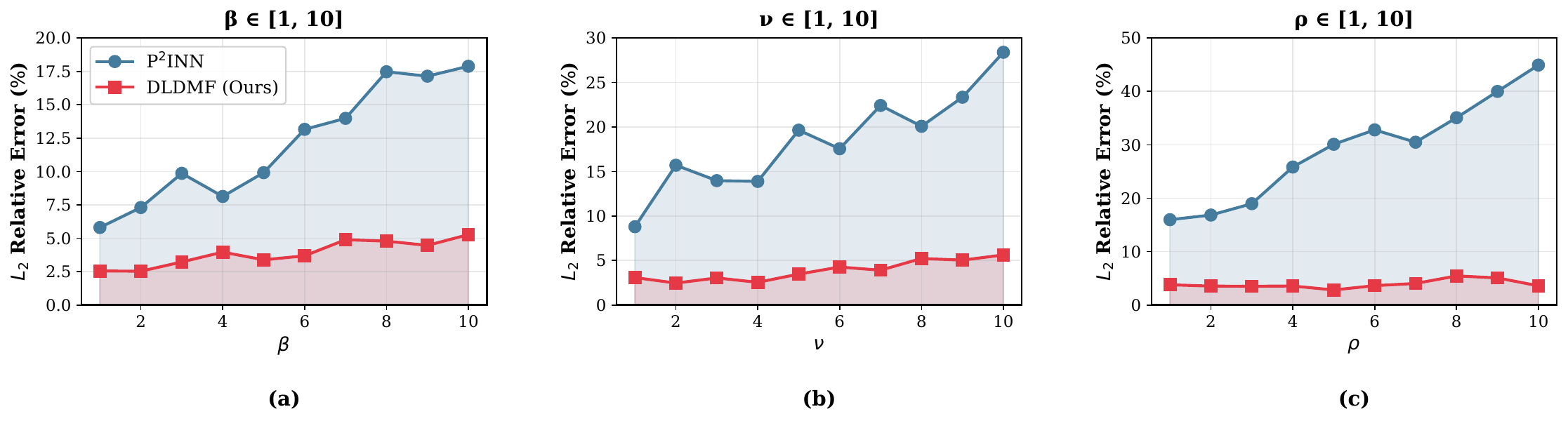}
\caption{CDR prediction errors across parameter sweeps.
(a) Convection ($\beta$). (b) Diffusion ($\nu$). (c) Reaction ($\rho$).}
\label{fig:param_generalization}
\end{figure}

\paragraph{\textbf{Multi-PDE comparison across CDR parameters.}}
The analysis in Fig.~\ref{fig:multi_pde_comparison} covers the pure convection, diffusion, and reaction equations, the convection--diffusion and reaction--diffusion couplings, and the full CDR system.
The comparison is important because prior analyses of CDR benchmarks show that different operator combinations induce distinct failure modes for physics-informed training, including sharp transport, smoothing-dominated diffusion, stiff reaction, and mixed CDR dynamics~\cite{krishn2021char}.
The comparison shows whether DLDMF's advantage persists when the governing operator changes across the selected PDE families.

Across all six PDE types, DLDMF achieved lower test relative $L_2$ error than P$^2$INN in the presented runs.
For the three single-operator cases, the error was reduced from 12.30\% to 3.50\% on convection, from 18.50\% to 4.20\% on diffusion, and from 25.60\% to 3.80\% on reaction.
The same trend held for coupled dynamics: DLDMF reduced the error from 15.20\% to 5.10\% on convection--diffusion, from 22.80\% to 6.30\% on reaction--diffusion, and from 32.10\% to 7.20\% on the full CDR equation.
The coupled cases yield higher DLDMF errors than the single-operator cases, showing that its relative accuracy advantage is maintained as the prediction task becomes more demanding.
Equivalently, the relative reduction ranges from 66.40\% to 85.20\%, with the largest gains appearing in reaction and full CDR regimes where the baseline error is highest.

The multi-PDE results lead to two observations.
First, the lower DLDMF errors extend from parameter variation within one PDE to changes in the balance between convection, diffusion, and reaction.
Second, the largest relative reductions occur for the reaction and full CDR cases, which have the highest P$^2$INN errors in the comparison.
Fig.~\ref{fig:multi_pde_comparison} provides cross-equation evidence that complements Fig.~\ref{fig:param_generalization}.
The former varies the PDE type, whereas the latter varies parameters within selected PDE families.

\begin{figure}[H]
\centering
\includegraphics[width=0.95\linewidth]{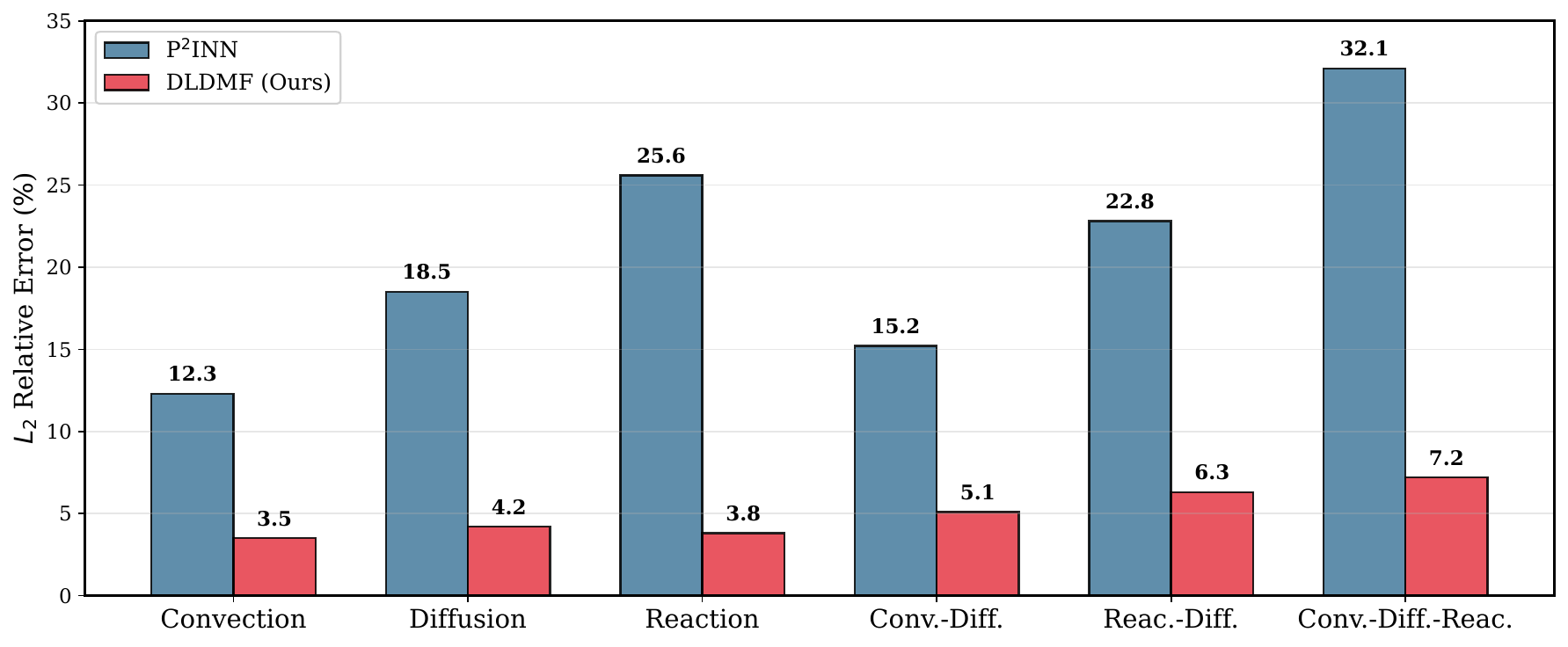}
\caption{Multi-PDE comparison across pure and coupled CDR test benchmarks.}
\label{fig:multi_pde_comparison}
\end{figure}

\paragraph{\textbf{Generalization beyond CDR benchmarks.}}
Fig.~\ref{fig:navier_stokes_2d} presents the error comparison on a two-dimensional incompressible Navier--Stokes flow benchmark, whose temporal prediction requires coherent spatial structures despite rapid error accumulation over time~\cite{li2020fourier,yin2023dino,wang2025pido}.
The Reynolds number varies from $\mathrm{Re}=100$ to $\mathrm{Re}=5000$, so the implementation probes increasingly challenging regimes in which inertial effects dominate more strongly over viscous smoothing.

Across all four Reynolds numbers, DLDMF produced lower relative $L_2$ error curves with less temporal variation than P$^2$INN.
At $\mathrm{Re}=100$, the mean error decreased from 5.84\% for P$^2$INN to 1.49\% for DLDMF.
At $\mathrm{Re}=500$ and $\mathrm{Re}=1000$, the corresponding reductions were from 9.17\% to 1.98\% and from 12.96\% to 2.58\%, respectively.
The most challenging displayed case, $\mathrm{Re}=5000$, shows the sharpest contrast: P$^2$INN reaches the plotted error ceiling of 40.00\%, whereas DLDMF remains around 6.52\% on average and below 7.55\% over the full time window.
The increase in DLDMF's mean error with Reynolds number characterizes its sensitivity to reduced viscous smoothing within the evaluated incompressible-flow family.

\begin{table}[t]
\centering
\small
\caption{Two-dimensional Navier--Stokes mean errors across Reynolds numbers.}
\label{tab:navier_stokes_2d_metrics}
\begin{tabular}{cccc}
\toprule
Reynolds number & P$^2$INN rel. $L_2$ (\%) & DLDMF rel. $L_2$ (\%) & Relative reduction (\%) \\
\midrule
100 & 5.84 & \textbf{1.49} & \textbf{74.49} \\
500 & 9.17 & \textbf{1.98} & \textbf{78.41} \\
1000 & 12.96 & \textbf{2.58} & \textbf{80.09} \\
5000 & $\ge 40.00$ & \textbf{6.52} & $\ge \textbf{83.70}$ \\
\bottomrule
\end{tabular}
\par\vspace{2pt}
\scriptsize\textit{Note:} At $\mathrm{Re}=5000$, P$^2$INN errors rapidly exceed the 40.00\% relative error threshold; the lower bound reflects this threshold.
\end{table}

\begin{figure}[H]
\centering
\includegraphics[width=0.95\linewidth]{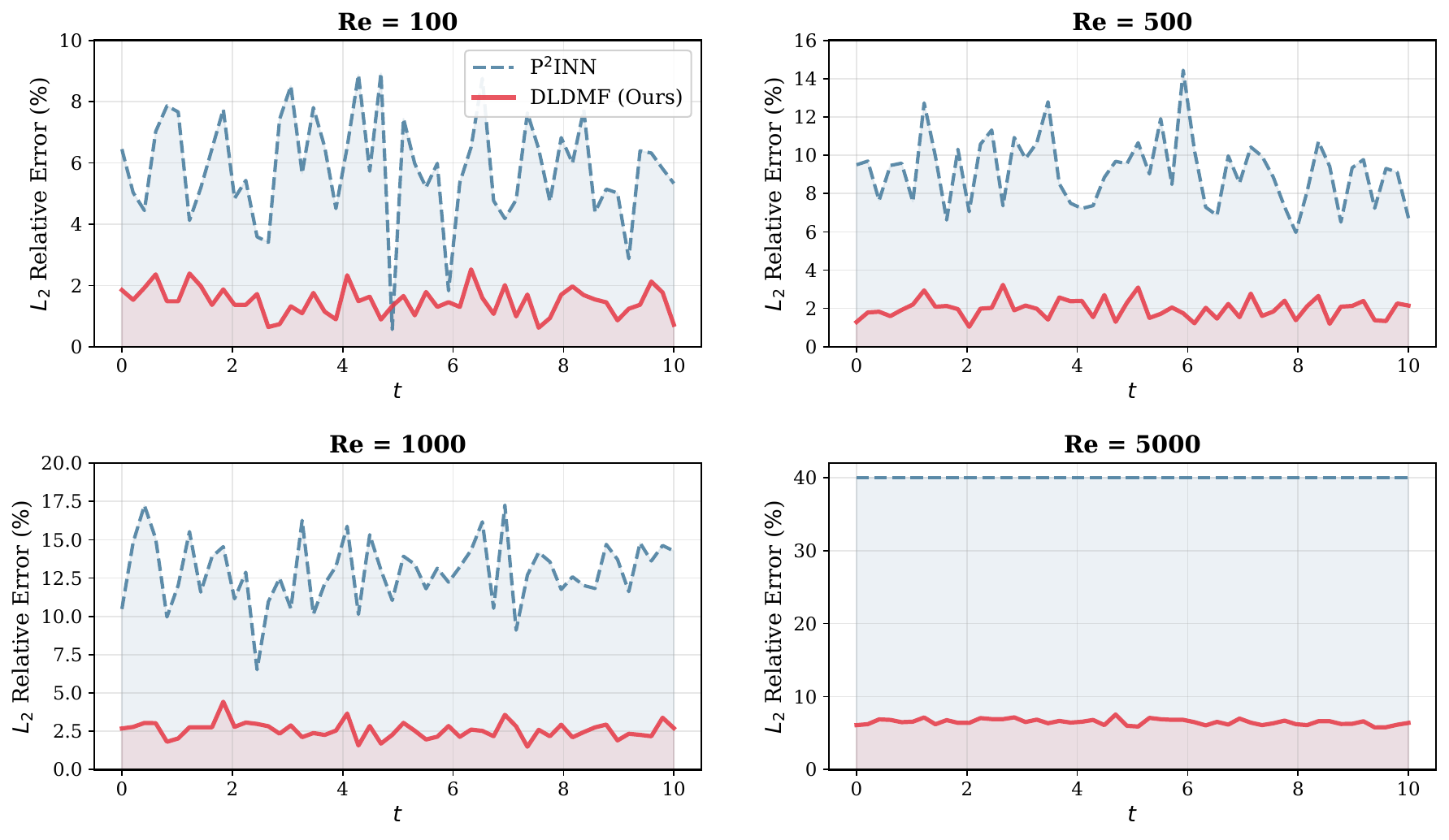}
\caption{2D Navier--Stokes comparison across Reynolds-number regimes.}
\label{fig:navier_stokes_2d}
\end{figure}

Fig.~\ref{fig:helmholtz_parameter_extrapolation_fields} compares the stationary solution fields with respect to different frequency parameters.
The upper panel contains the in-parameter pairs with $a_1,a_2\le5$.
DLDMF closely matches the reference nodal patterns across the displayed cases.
P$^2$INN and PIDO exhibit increasing phase and amplitude distortion at higher frequencies.
The lower panel presents out-of-parameter pairs with $\max(a_1,a_2)>5$.
Both baselines lose the high-frequency oscillatory structure, whereas DLDMF preserves the nodal layout, frequency, and amplitude across the displayed extrapolation cases.
The results provide qualitative evidence of improved parameter extrapolation for stationary PDE fields.

\begin{figure}[H]
\centering
\includegraphics[width=0.92\linewidth]{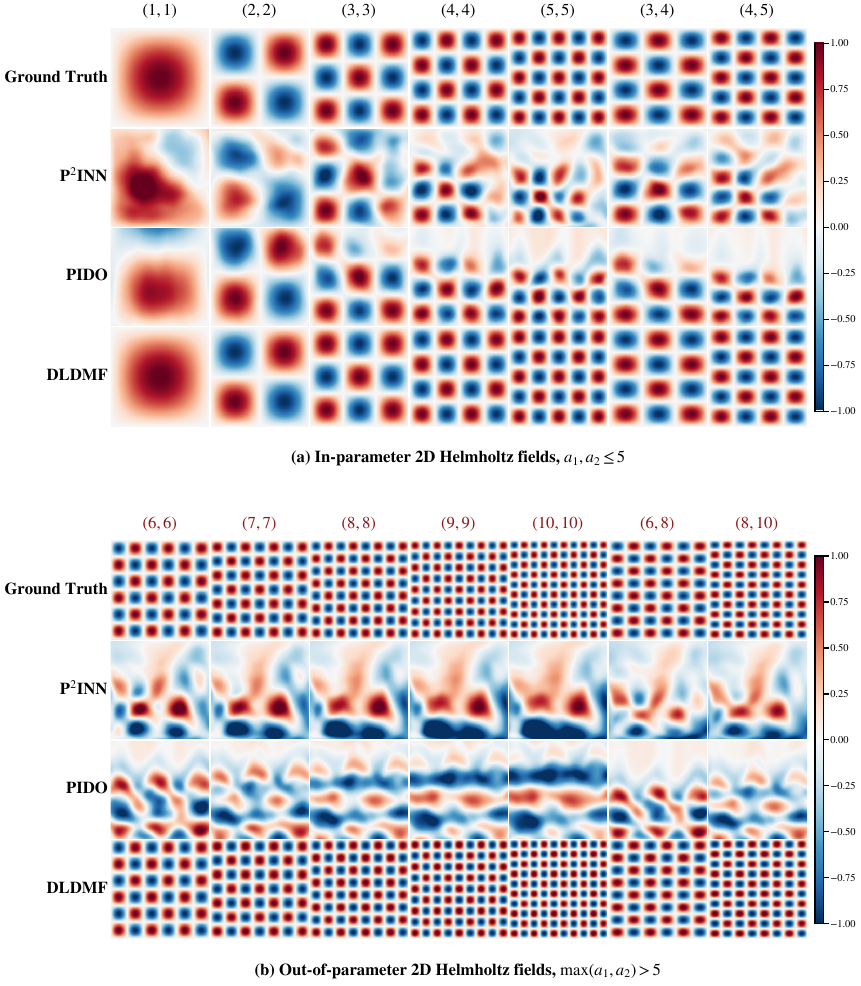}
\caption{Stationary Helmholtz fields across parameter regimes.
(a) In-parameter, $a_1,a_2\leq5$. (b) Out-of-parameter, $\max(a_1,a_2)>5$.}
\label{fig:helmholtz_parameter_extrapolation_fields}
\end{figure}

\paragraph{\textbf{Latent-dynamics drift under time extrapolation.}}
Fig.~\ref{fig:dldmf_latent_dyn_drift} displays the evolution of each model's internal representation beyond the training horizon $T_{\mathrm{tr}}=1.0$, complementing the parameter-wise comparison in Fig.~\ref{fig:param_generalization} with a temporal view.
For a fixed in-distribution diffusion parameter ($\nu=3$), we track the P$^2$INN's coordinate embedding and DLDMF's conditioned latent state at a fixed spatial location over $t\in[0,2]$, together with the corresponding relative $L_2$ prediction error.
The two models diverge sharply on the error curve (panel (a)).
P$^2$INN's error rises quickly after $t=0$, peaks at approximately 18.6\% around $t\approx0.5$, remains above 14.00\% at the train/test boundary ($t=1.0$), and recedes to 3.55\% by $t=2.0$.
The later decrease coincides with the displayed output field and is not interpreted as recovery of the internal representation.
DLDMF's error stays essentially flat across the same window, with in-time and out-time means of 0.18\% and 0.17\%, respectively, and a value of 0.15\% at $t=1.0$.
The close In-$t$/Out-$t$ agreement indicates consistent output accuracy across the training boundary, alongside the continued evolution of the learned latent state.
The values are approximately two orders of magnitude smaller than the corresponding P$^2$INN errors over most of the displayed interval.
Panels (b)--(d) show three randomly sampled dimensions of each representation.
The P$^2$INN coordinate dimension shown in panel (b) remains within $[-0.44,0.69]$, whereas the corresponding displayed DLDMF latent coordinate spans $[0.08,1.90]$ and continues to vary beyond $t=1.0$.
Because separately trained latent coordinates are model-specific, the latent-trajectory panels illustrate the temporal evolution within each learned representation.
The prediction-error curve provides the direct performance comparison.

\begin{table}[t]
\centering
\small
\caption{Latent-dynamics drift metrics under time extrapolation.}
\label{tab:latent_drift_metrics}
\begin{tabular}{ccccc}
\toprule
Model & Split & Mean rel. $L_2$ (\%) & Max rel. $L_2$ (\%) & Boundary/final rel. $L_2$ (\%) \\
\midrule
P$^2$INN & In-$t$ & 14.62 & 18.55 & 14.74 \\
P$^2$INN & Out-$t$ & 8.01 & 14.06 & 3.55 \\
DLDMF & In-$t$ & \textbf{0.18} & \textbf{0.36} & \textbf{0.15} \\
DLDMF & Out-$t$ & \textbf{0.17} & \textbf{0.27} & \textbf{0.27} \\
\bottomrule
\end{tabular}
\end{table}

\begin{figure}[H]
\centering
\includegraphics[width=0.95\linewidth]{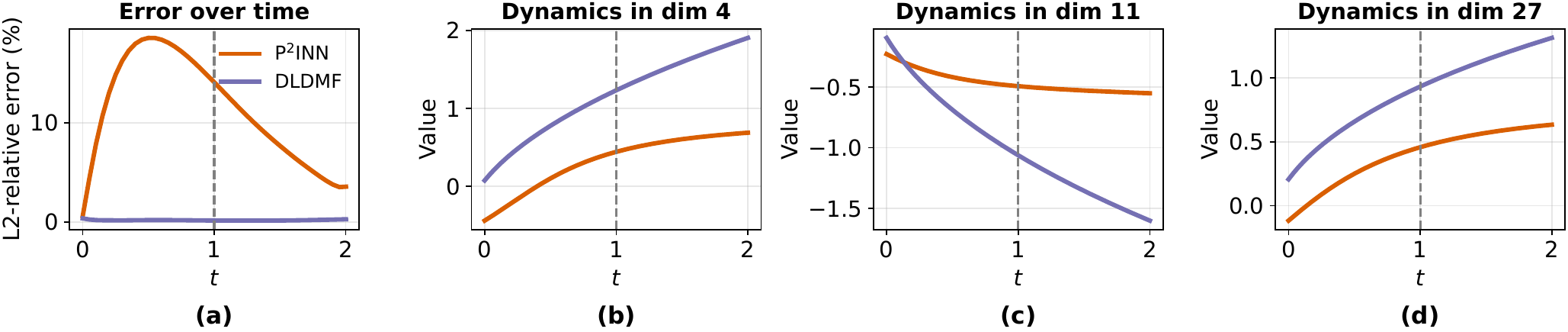}
\caption{Prediction error and dynamics for P$^2$INN and DLDMF during temporal extrapolation.
(a) Relative $L_2$ error over time. (b) Dynamics in dimension 4. (c) Dynamics in dimension 11. (d) Dynamics in dimension 27.}
\label{fig:dldmf_latent_dyn_drift}
\end{figure}

\FloatBarrier

\section{Conclusions and outlook}
\label{sec:conclusion}
This work proposes Disentangled Latent Dynamics Manifold Fusion (DLDMF), a physics-informed latent time-integration surrogate for the parameterized temporal extrapolation of PDEs.
In the DLDMF framework, the parameter embedding is designed to determine the latent initial state and to condition the subsequent vector field, while the nonlinear decoder is developed to reconstruct the physical field from spatial, parameter, and temporal representations.
The resulting formulation replaces the traditional direct extrapolation of the time coordinate with parameter-conditioned ODE integration and removes the instance-wise latent optimization required by auto-decoding approaches.

The theoretical analysis provides a conditional finite-horizon interpretation of the DLDMF framework through parameter-perturbation, latent-flow, reconstruction-error, and numerical-integration bounds.
The resulting estimates identify the parameter-dependent initialization, latent-vector-field stability, reconstruction regularity, and ODE-solver accuracy as quantities governing the finite-horizon prediction.
The CDR and Navier--Stokes experiments show lower errors for DLDMF under the evaluated settings; the Helmholtz field comparisons provide qualitative evidence of parameter generalization.

Further development should address parameter-space coverage, physical structure preservation, and solver-dependent computational cost.
Promising directions also include boundary-aware parameter sampling, regularization of parameter sensitivity, and stiffness-aware or structure-preserving latent integrators.
Future investigations should also address temporal prediction under severe parameter shifts and systematically quantify variability across independently trained models.

\section*{Data and code availability}
The benchmark data in the study are generated from the PDE specifications and numerical solvers described in the manuscript.
The anonymized source code, data-generation scripts, configuration files, reference-solver settings, trained-checkpoint metadata, and evaluation scripts will be made available for peer review through an anonymous repository link at submission.
The repository is intended to contain one command for reproducing each presented table and figure.
A DOI-linked archival version will be deposited upon acceptance.

\section*{CRediT authorship contribution statement}
Zhangyong Liang: Conceptualization, Methodology, Software, Formal analysis, Investigation, Visualization, Writing -- original draft.
Jihong Zhu: Supervision, Validation, Writing -- review \& editing.
Huanhuan Gao: Conceptualization, Methodology, Supervision, Validation, Writing -- review \& editing.

\section*{Declaration of competing interest}
The authors declare that they have no known competing financial interests or personal relationships that could have appeared to influence the work described in the paper.

\section*{Funding}
The work received funding from the National Natural Science Foundation of China [grant number 12572138].
The funder had no role in study design, data collection and analysis, decision to publish, or preparation of the manuscript.

\section*{Ethics approval}
The computational study does not involve human participants, human data, or animal experiments.


\bibliographystyle{elsarticle-num-names}

\bibliography{cas-refs}

\end{document}